%% file: main_arxiv.tex
\documentclass[]{fairmeta}

\input{preamble}

\title{Off-Context GRPO: Learning to Reason on Hard Problems using Privileged Information}

\author[1,2,*]{Priyank Agrawal}
\author[1,2]{Ankur Samanta}
\author[1]{Shervin Ghasemlou}
\author[1]{Boris Vidolov}
\author[1]{Jalaj Bhandari}
\author[1,\dagger]{Kavosh~Asadi}
\author[1,\dagger]{Daniel Jiang}
\author[1,\dagger]{Aditya Modi}

\affiliation[1]{Meta AI}
\affiliation[2]{Columbia University}
\contribution[*]{Work done at Meta}
\contribution[\dagger]{Joint last author}

\abstract{\input{abstract}}

\date{\today}
\correspondence{Priyank Agrawal at \email{pa2608@columbia.edu}}

\metadata[Code]{\url{https://github.com/AgPriyank/OC-GRPO}}

\begin{document}

\maketitle

\input{introduction}

\input{preliminaries}
\input{overcoming_learning_cliff}

\input{off_context_GRPO}
\input{empirical_results}
\input{conclusion}

\clearpage
\newpage
\bibliographystyle{assets/plainnat}
\bibliography{ref}

\input{appendix.tex}

\end{document}

%% file: preamble.tex
\usepackage{color}
\usepackage{xcolor}
\usepackage{amsthm,amssymb}
\usepackage{amsmath}
\usepackage{comment}
\usepackage{commath}
\usepackage{enumitem}
\usepackage{mathtools}
\usepackage{multirow}
\usepackage{wrapfig}
\usepackage{pifont}
\usepackage[most]{tcolorbox}

\usepackage{nicefrac}
\usepackage{amsfonts}
\usepackage{commath}
\usepackage{esvect}

\usepackage[ruled,vlined]{algorithm2e}

\SetCommentSty{emph}
\SetAlFnt{\small}
\SetKwComment{tcp}{\hfill\(\triangleright\) }{}

\usepackage{natbib}
\usepackage{makecell}
\usepackage{cancel}

\usepackage{colortbl}

\newcommand{\pa}[1]{\textcolor{olive}{#1}}

\usepackage{microtype}
\usepackage{graphicx}
\usepackage{subcaption}
\usepackage{booktabs} 
\usepackage{multirow}
\usepackage{comment}
\usepackage{thmtools} 
\usepackage{thm-restate}

\usepackage{hyperref}
\usepackage{lineno}

\usepackage{xcolor}         
\usepackage[utf8]{inputenc} 
\usepackage[T1]{fontenc}    
\usepackage{hyperref}       
\usepackage{url}            
\usepackage{booktabs}       
\usepackage{amsfonts}       
\usepackage{nicefrac}       
\usepackage{microtype}      
\usepackage{comment}

\usepackage{amsmath}
\usepackage{amssymb}
\usepackage{mathtools}
\usepackage{amsthm}
\usepackage{makecell}

\usepackage{float}

\usepackage{longtable}
\usepackage{comment}

\usepackage{arydshln}

\usepackage{rotating}

\usepackage{array, multirow, bigdelim, booktabs} 
\usepackage{makecell}
\usepackage{mathtools}
\usepackage{commath}
\usepackage{xcolor,color}
\newcommand{\toberemoved}[1]{}
\newcommand{\algoname}{\ensuremath{\mathtt{OC\text{-}GRPO}}}

\newcommand{\algonameAdaptive}{\ensuremath{\mathtt{OC\text{-}GRPO\text{-}Adaptive}}}
\newcommand{\algonameFixed}{\ensuremath{\mathtt{OC\text{-}GRPO\text{-}Fixed}}}
\newcommand{\algonameOn}{\algonameAdaptive}
\newcommand{\algonameOff}{\algonameFixed}

\newtheorem{lemma}{Lemma}

\newtheorem{theorem}{Theorem}
\newtheorem{assumption}{Assumption}

\theoremstyle{plain}

\theoremstyle{definition}

\theoremstyle{plain}
\newtheorem{remark}[theorem]{Remark}

\theoremstyle{plain}
\newtheorem{example}{Example}


\usepackage[textsize=tiny]{todonotes}

\usepackage{amsmath,amssymb}

\usepackage{xcolor}

\usepackage{bbm}        

\usepackage{listings}

\usepackage[most]{tcolorbox} 
\tcbuselibrary{listings,breakable,skins}

\newtcolorbox{hypobox}{
  colback=blue!4!white,
  colframe=blue!70!black,
  fonttitle=\bfseries,
  title=Hypothesis,
  breakable
}

\newtcblisting{promptbox}[1][]{%
  colback=gray!3!white,
  colframe=gray!60!black,
  listing only,
  breakable,
  left=1mm,right=1mm,top=1mm,bottom=1mm,
  listing options={basicstyle=\ttfamily\footnotesize,breaklines=true},
  title=#1
}

\usepackage{booktabs}
\usepackage{siunitx}
\usepackage{threeparttable} 
 \usepackage{multirow}
 \usepackage{rotating} 

  \usepackage[table]{xcolor}
 \usepackage{booktabs,tabularx,makecell,array}
 \newcolumntype{Y}{>{\raggedright\arraybackslash}X}

\definecolor{bestgreen}{RGB}{198,239,206}
\definecolor{goodgreen}{RGB}{226,239,218}
\definecolor{mixedyellow}{RGB}{255,242,204}
\definecolor{lateorange}{RGB}{252,228,214}
\definecolor{failred}{RGB}{244,204,204}
\definecolor{sectiongray}{RGB}{242,242,242}

\definecolor{softgreen}{RGB}{226,239,218}
\definecolor{softyellow}{RGB}{255,242,204}
\definecolor{softred}{RGB}{248,215,218}
\definecolor{softgray}{RGB}{242,242,242}

\newcommand{\rescell}[2]{\makecell[c]{#1\\[-0.05em]{\tiny #2}\rule[-0.48em]{0pt}{0.48em}}}

\usepackage{tikz}
\usetikzlibrary{positioning, fit, backgrounds, shapes.misc, calc}

\newcommand{\duringbest}[2]{\rescell{\textit{#1}}{\textit{#2}}}

\usepackage[utf8]{inputenc} 
\usepackage[T1]{fontenc}    
\usepackage{hyperref}       
\usepackage{url}            
\usepackage{booktabs}       
\usepackage{amsfonts}       
\usepackage{nicefrac}       
\usepackage{microtype}      
\usepackage{xcolor}         
\usepackage{tikz}
\usetikzlibrary{arrows.meta, decorations.pathreplacing, shapes.multipart}
\usepackage{pifont}

\definecolor{cg}{RGB}{76,175,80}
\definecolor{cr}{RGB}{239,83,80}
\definecolor{cb}{RGB}{66,133,244}
\definecolor{cy}{RGB}{255,128,0}

%% file: abstract.tex
Reinforcement learning with verifiable rewards (RLVR) improves reasoning in large language models. Yet, typical RLVR approaches fail on difficult problems: when a model cannot generate any correct solutions, it receives \textit{zero} learning signal. Providing privileged guidance during training, such as solution prefixes, can help overcome this learning cliff by steering the model towards {correct solutions with non-zero reward}. {We call these rollouts \textit{off-context}: they are generated from a training prompt that contains privileged guidance, while the target objective is defined by the original prompt without that guidance.} {We introduce} Off-Context GRPO (OC-GRPO), a minimally modified variant of GRPO that uses guided rollouts but applies an importance-corrected objective to steer the update back toward the original unguided objective, avoiding the mismatch that destabilizes uncorrected guided training. Empirically, our algorithm achieves a 3.9\% absolute improvement (13.8\% relative gain) over vanilla GRPO on average across standard mathematical reasoning benchmarks with negligible additional cost.

%% file: introduction.tex

\section{Introduction}

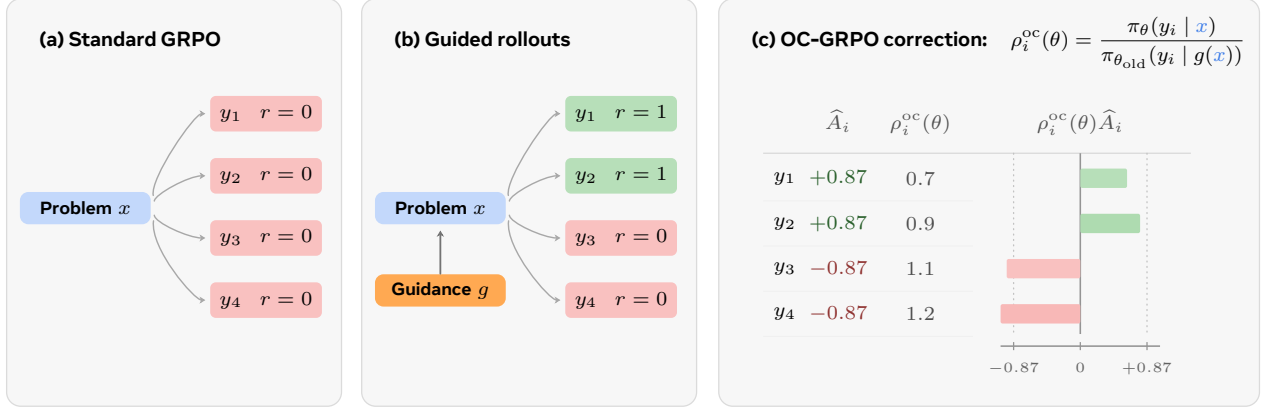
\begin{figure*}
\centering
\resizebox{\textwidth}{!}{%
\begin{tikzpicture}[>=Stealth, font=\sffamily]

\def\rone{1.35}
\def\rtwo{0.58}
\def\rthree{-0.19}
\def\rfour{-0.96}
\def\ptop{2.55}
\def\pbot{-2.05}
\def\titleInset{0.28}
\def\titleY{2.50}

\tikzset{
  panel/.style={draw=black!14, fill=black!3, rounded corners=5pt, line width=0.45pt},
  ptitle/.style={font=\sffamily\bfseries\footnotesize, anchor=north west},
  prompt/.style={draw=none, fill=cb!32, rounded corners=3pt,
                 font=\sffamily\scriptsize, align=center,
                 text width=1.48cm, minimum height=0.40cm, inner sep=2pt},
  guidance/.style={draw=none, fill=cy!68, rounded corners=3pt,
                   font=\sffamily\scriptsize, align=center,
                   text width=1.48cm, minimum height=0.40cm, inner sep=2pt},
  gprompt/.style={rectangle split, rectangle split parts=2,
                  rectangle split part fill={cb!6, cy!9},
                  draw=none, rounded corners=3pt,
                  font=\sffamily\scriptsize, align=left, inner sep=5pt, text width=1.52cm},
  good/.style={draw=none, fill=cg!40, rounded corners=2pt,
               minimum width=1.35cm, minimum height=0.36cm,
               font=\sffamily\scriptsize},
  bad/.style={draw=none, fill=cr!36, rounded corners=2pt,
              minimum width=1.35cm, minimum height=0.36cm,
              font=\sffamily\scriptsize},
  smallhead/.style={font=\sffamily\scriptsize\bfseries, text=black!68},
  flow/.style={-{Stealth[length=2.1pt,width=2.5pt]}, black!34, line width=0.50pt,
               line cap=round, line join=round, shorten >=1.5pt, shorten <=1.5pt}
}

\def\ax{0}
\node[panel, minimum width=4.15cm, minimum height=5.05cm, anchor=center] (paPanel) at (\ax, 0.25) {};
\node[ptitle] at ({\ax-2.075+\titleInset}, \titleY) {(a) Standard GRPO};

\node[prompt]
  (pa) at ({\ax-1.10}, 0.18) {\textbf{Problem} $x$};
\node[bad] (a1) at ({\ax+1.14}, \rone) {$y_1\quad r=0$};
\node[bad] (a2) at ({\ax+1.14}, \rtwo) {$y_2\quad r=0$};
\node[bad] (a3) at ({\ax+1.14}, \rthree) {$y_3\quad r=0$};
\node[bad] (a4) at ({\ax+1.14}, \rfour) {$y_4\quad r=0$};

\draw[flow] ([xshift=1pt,yshift=3pt]pa.east)
  .. controls ({\ax-0.30}, 0.55) and ({\ax+0.25}, \rone) .. (a1.west);
\draw[flow] ([xshift=1pt,yshift=1pt]pa.east)
  .. controls ({\ax-0.24}, 0.35) and ({\ax+0.18}, \rtwo) .. (a2.west);
\draw[flow] ([xshift=1pt,yshift=-1pt]pa.east)
  .. controls ({\ax-0.24}, 0.00) and ({\ax+0.18}, \rthree) .. (a3.west);
\draw[flow] ([xshift=1pt,yshift=-3pt]pa.east)
  .. controls ({\ax-0.30}, -0.32) and ({\ax+0.25}, \rfour) .. (a4.west);

\def\bx{4.40}
\node[panel, minimum width=4.15cm, minimum height=5.05cm, anchor=center] (pbPanel) at (\bx, 0.25) {};
\node[ptitle] at ({\bx-2.075+\titleInset}, \titleY) {(b) Guided rollouts};

\node[prompt]
  (pb) at ({\bx-1.10}, 0.18) {\textbf{Problem} $x$};
\node[guidance]
  (pg) at ({\bx-1.10}, -0.84) {\textbf{Guidance} $g$};
\draw[-{Stealth[length=2.25pt,width=2.65pt]}, black!55, line width=0.64pt, line cap=round, line join=round]
  (pg.north) -- ([yshift=-2pt]pb.south);
\node[good] (b1) at ({\bx+1.14}, \rone) {$y_1\quad r=1$};
\node[good] (b2) at ({\bx+1.14}, \rtwo) {$y_2\quad r=1$};
\node[bad]  (b3) at ({\bx+1.14}, \rthree) {$y_3\quad r=0$};
\node[bad]  (b4) at ({\bx+1.14}, \rfour) {$y_4\quad r=0$};

\draw[flow] ([xshift=1pt,yshift=3pt]pb.east)
  .. controls ({\bx-0.30}, 0.55) and ({\bx+0.25}, \rone) .. (b1.west);
\draw[flow] ([xshift=1pt,yshift=1pt]pb.east)
  .. controls ({\bx-0.24}, 0.35) and ({\bx+0.18}, \rtwo) .. (b2.west);
\draw[flow] ([xshift=1pt,yshift=-1pt]pb.east)
  .. controls ({\bx-0.24}, 0.00) and ({\bx+0.18}, \rthree) .. (b3.west);
\draw[flow] ([xshift=1pt,yshift=-3pt]pb.east)
  .. controls ({\bx-0.30}, -0.32) and ({\bx+0.25}, \rfour) .. (b4.west);

\def\cx{9.35}
\def\pcx{\cx+0.75}
\def\zeroX{11.23}
\def\barscale{0.95}
\def\barh{0.12}
\def\cone{0.55}
\def\ctwo{-0.01}
\def\cthree{-0.57}
\def\cfour{-1.13}
\def\cheadY{1.16}
\def\plotTop{0.86}
\def\plotBottom{-1.38}
\def\rowSepOffset{0.28}
\def\axisY{-1.53}
\node[panel, minimum width=6.70cm, minimum height=5.05cm, anchor=center] (pcPanel) at (\pcx, 0.25) {};
\node[ptitle] (pcTitle)
  at ({\pcx-3.35+\titleInset}, \titleY)
  {(c) OC-GRPO correction:};
\node[font=\sffamily\footnotesize, anchor=base west, inner sep=0pt, text=black]
  at ([xshift=0.16cm]pcTitle.base east)
  {\scalebox{0.9}{$\displaystyle \rho_i^{\mathrm{oc}}(\theta)=\frac{\pi_\theta(y_i\mid {\color{cb!95!black}x})}
    {\pi_{\theta_{\mathrm{old}}}(y_i\mid g({\color{cb!95!black}x}))}$}};

\node[smallhead, anchor=base] at ({\cx-1.12}, \cheadY) {$\widehat A_i$};
\node[smallhead, anchor=base] at ({\cx-0.10}, \cheadY) {$\rho_i^{\mathrm{oc}}(\theta)$};
\node[smallhead, anchor=base] at (\zeroX, \cheadY) {$\rho_i^{\mathrm{oc}}(\theta)\widehat A_i$};
\draw[black!15, line width=0.45pt] ({\cx-2.05}, \plotTop) -- ({\cx+3.25}, \plotTop);

\foreach \yy in {\cone,\ctwo,\cthree,\cfour} {
  \draw[black!6, line width=0.25pt] ({\cx-2.05}, {\yy-\rowSepOffset}) -- ({\cx+0.55}, {\yy-\rowSepOffset});
}

\node[font=\sffamily\scriptsize] at ({\cx-1.78}, \cone) {$y_1$};
\node[font=\sffamily\scriptsize] at ({\cx-1.78}, \ctwo) {$y_2$};
\node[font=\sffamily\scriptsize] at ({\cx-1.78}, \cthree) {$y_3$};
\node[font=\sffamily\scriptsize] at ({\cx-1.78}, \cfour) {$y_4$};

\node[font=\sffamily\scriptsize, text=cg!55!black] at ({\cx-1.12}, \cone) {$+0.87$};
\node[font=\sffamily\scriptsize, text=cg!55!black] at ({\cx-1.12}, \ctwo) {$+0.87$};
\node[font=\sffamily\scriptsize, text=cr!60!black] at ({\cx-1.12}, \cthree) {$-0.87$};
\node[font=\sffamily\scriptsize, text=cr!60!black] at ({\cx-1.12}, \cfour) {$-0.87$};

\node[font=\sffamily\scriptsize, text=black!70] at ({\cx-0.10}, \cone) {$0.7$};
\node[font=\sffamily\scriptsize, text=black!70] at ({\cx-0.10}, \ctwo) {$0.9$};
\node[font=\sffamily\scriptsize, text=black!70] at ({\cx-0.10}, \cthree) {$1.1$};
\node[font=\sffamily\scriptsize, text=black!70] at ({\cx-0.10}, \cfour) {$1.2$};

\draw[black!42, line width=0.55pt]
  (\zeroX, \plotTop) -- (\zeroX, \plotBottom);
\foreach \ref in {-0.87,0.87} {
  \draw[dash pattern=on 0.35pt off 1.55pt, line cap=round, black!24, line width=0.45pt]
    ({\zeroX+\ref*\barscale}, \plotTop) -- ({\zeroX+\ref*\barscale}, \plotBottom);
}

\path[draw=none, fill=cg!40, rounded corners=0.7pt]
  (\zeroX, {\cone-\barh}) rectangle ({\zeroX+0.61*\barscale}, {\cone+\barh});
\path[draw=none, fill=cg!45, rounded corners=0.7pt]
  (\zeroX, {\ctwo-\barh}) rectangle ({\zeroX+0.78*\barscale}, {\ctwo+\barh});
\path[draw=none, fill=cr!36, rounded corners=0.7pt]
  ({\zeroX-0.96*\barscale}, {\cthree-\barh}) rectangle (\zeroX, {\cthree+\barh});
\path[draw=none, fill=cr!41, rounded corners=0.7pt]
  ({\zeroX-1.04*\barscale}, {\cfour-\barh}) rectangle (\zeroX, {\cfour+\barh});

\draw[black!42, line width=0.55pt]
  ({\zeroX-1.04*\barscale}, \axisY) -- ({\zeroX+1.04*\barscale}, \axisY);
\foreach \tick/\lab in {-0.87/-0.87,0/0,0.87/+0.87} {
  \draw[black!42, line width=0.45pt]
    ({\zeroX+\tick*\barscale}, {\axisY-0.05}) -- ({\zeroX+\tick*\barscale}, {\axisY+0.05});
  \node[font=\sffamily\tiny, anchor=north, text=black!62] at ({\zeroX+\tick*\barscale}, {\axisY-0.08}) {$\lab$};
}

\end{tikzpicture}%
}
    \caption{\textbf{Off-Context GRPO.} \textbf{(a)}~On hard problems, unguided samples all receive reward $0$, so the group-relative advantages vanish. \textbf{(b)}~Guidance is supplied with the problem and produces successful rollouts, while the update still targets the unguided objective at $x$. \textbf{(c)}~\algoname{} keeps the guided rollouts and reweights their advantages with the importance weight shown in the panel. Successes made likely by guidance receive smaller positive contributions; failures receive larger negative contributions.}
    \label{fig:ocgrpo_schematic}
    \label{fig:oc-grpo-overview}
\end{figure*}
Reinforcement learning with verifiable rewards (RLVR) has emerged as a powerful approach for improving reasoning in large language models~\citep{guo2025deepseek,yu2025dapo,jaech2024openai}. The dominant algorithm for RLVR is Group Relative Policy Optimization (GRPO)~\citep{shao2024deepseekmath} which itself is based on the seminal PPO algorithm \citep{schulman2017proximal}. In GRPO, for each problem, we first sample a group of responses from an LLM, use a verifier to score each response, and then compute in-group advantages by normalizing rewards within the group. These advantage estimates are then used to update the LLM policy via a PPO-like clipped importance-sampling surrogate. The gradient signal in this GRPO update comes entirely from within-group reward variance, i.e., GRPO learns only when some rollouts in a group succeed and others fail, reinforcing the successful ones relative to the failures.

This requirement exposes a fundamental failure mode on hard problems. When every rollout in a group fails, rewards are uniformly zero, within-group variance collapses, and the gradient is identically zero, no matter how long training continues. We refer to this as the \textit{learning cliff}~\citep{guo2025g,zhang2025scaf,qu2026pope,yanlearning}. Escaping this cliff requires at least some successes within each rollout group. One workaround is to use oracle/expert solutions directly as training targets, via supervised fine-tuning or off-policy RL~\citep{yanlearning,song2024trial,agarwal2024policy,huang2025blending,zhang2025stephint,zhangbread}. This breaks the cliff by construction but at a cost: it distorts the base model's native reasoning style and requires heavy stabilization machinery such as reward shaping, entropy control, and hybrid SFT--RL objectives. A second class of methods takes a more surgical approach, using oracle information not as a training target but as \emph{context}, injecting privileged guidance into the prompt to make correct continuations reachable, while the model still samples its own rollouts. Examples include solution-prefix conditioning~\citep{qu2026pope,zhangbread,li2025staying}, output-space prefix injection~\citep{setlur2026reuse,amani2025rl}, and tiered in-prompt hints~\citep{zhang2025scaf}. This preserves GRPO's on-policy structure while breaking the cliff. However, all of these methods share a subtle but consequential flaw we call the \textit{off-context} problem: rollouts are sampled under a guided prompt the model never sees at deployment, yet gradients are computed as if the sampling and evaluation distributions matched. The result is either a misaligned objective or biased gradients that can destabilize training.

In this work, we propose~\textbf{\algoname{}: Off-context GRPO}, a minimal modification
to GRPO (Figure~\ref{fig:oc-grpo-overview}). \algoname{}~re-weights per-token advantages by the importance ratio between the unguided and guided policies, a standard technique used for correction for off-distribution sampling~\citep{precup2001off, degris2012off}. The resulting estimator is unbiased for the original unguided objective, and its variance scales with the length of the guidance prefix rather than the rollout, yielding a clear principle: use the shortest guidance that breaks the cliff. The framework is mechanism-agnostic, applying to any form of privileged signal since what matters is the context mismatch, not its form.

Beyond fixing the objective, the importance ratio encodes a meaningful inductive bias about how much credit each trajectory generated under the guided context deserves. As illustrated in Figure~\ref{fig:oc-grpo-overview}, the correction dampens the credit assigned to correct trajectories, and the dampening is stronger for solution trajectories that rely more heavily on the guidance, limiting the influence of successes the model could not have produced on its own. We formalize this intuition in Section~\ref{sec:understanding_oc_grpo}. Our contributions are as follows.

\begin{itemize}
    \item We formalize privileged guidance in RLVR as off-context sampling and show that existing methods implicitly optimize a different objective than the one evaluated at deployment.

    \item We propose \algoname{}, a minimal modification to GRPO that corrects for the off-context mismatch via importance sampling, with provable unbiasedness and variance controlled by guidance length.

    \item We prove that the importance correction induces a behavior-aware credit assignment (Theorem~\ref{thm:credit_assignment}): a guided success retains credit only to the extent that the model could have produced it without the guidance, while a failure that persists despite the guidance draws an amplified penalty. The update therefore reinforces genuine problem solving rather than hint-following, without any additional reward shaping.\looseness-1

    \item We empirically validate the proposed method on multiple models and benchmarks. On Qwen2.5-7B-Instruct, \algoname{} delivers a 13.8\% relative gain over GRPO on math reasoning. At 3B and 1.5B scale, guided-target baselines degrade below vanilla GRPO while \algoname{} retains consistent gains, evidence that back-generalization is capacity-dependent and the off-context correction matters most at smaller scales.\looseness-1
\end{itemize}

%% file: preliminaries.tex
\section{Preliminaries: RLVR and GRPO}
\label{sec:prelim}

RLVR trains a policy on problems with automatically checkable answers~\citep{guo2025deepseek}. A training distribution $\mathcal{D}$ supplies problems $x \in \mathcal{X}$, and the policy samples a response $y=(y_1,\ldots,y_T)$ autoregressively from $\pi_\theta(\cdot \mid x)$, where $T$ denotes the response length. We use the base model $\pi_{\mathrm{ref}}$ as the KL reference. A response may contain a reasoning trace, but the verifier scores only the extracted final answer: we write $a(y)$ for the answer extracted from $y$ and $a^\star(x)$ for the correct answer to $x$. The verifier returns the binary reward $r(x,y)=\mathbf{1}\{a(y)=a^\star(x)\}$. The RLVR objective is the expected verifier reward:
\begin{equation}
\label{eq:rl_objective}
J(\theta)
\;=\;
\mathbb{E}_{x \sim \mathcal{D}}\;
\mathbb{E}_{y \sim \pi_\theta(\cdot \mid x)}\!
\bigl[r(x, y)\bigr].
\end{equation}

GRPO converts these binary rewards into relative advantages within a rollout group~\citep{shao2024deepseekmath}. For each problem $x$, the behavior policy $\pi_{\theta_{\mathrm{old}}}$ generates $G$ responses $\{y_i\}_{i=1}^{G}$. We write $T_i$ for the length of $y_i$. Let $\mu_r$ and $\sigma_r$ be the mean and standard deviation of the group rewards $\{r(x,y_i)\}_{i=1}^G$. GRPO assigns
$\widehat{A}_i=(r(x,y_i)-\mu_r)/\sigma_r$. When $\sigma_r=0$, the group contains no preference information, and the usual convention sets all advantages to zero. The policy is updated by optimizing the clipped, KL-regularized surrogate averaged over training problems:
\begin{equation}
\label{eq:grpo}
\begin{aligned}
\mathcal{L}_{\mathrm{GRPO}}(\theta)
&=
\mathbb{E}_{x \sim \mathcal{D}}
\left[
\frac{1}{G}\sum_{i=1}^{G}
\frac{1}{T_i}\sum_{t=1}^{T_i}
\left\{
\min\!\left(
\rho_{i,t}(\theta)\widehat{A}_i,\,
\mathrm{clip}\!\left(\rho_{i,t}(\theta), 1{-}\epsilon, 1{+}\epsilon\right)\widehat{A}_i
\right)
- \beta\,\mathrm{KL}_{i,t}(\theta)
\right\}
\right].
\end{aligned}
\end{equation}
The per-token PPO ratio~\citep{schulman2017proximal} compares the current policy with the behavior policy that produced the rollout, and the KL term penalizes drift from the reference model:
\begin{equation}\label{eq:grpo_ratio}
\rho_{i,t}(\theta)
:=
\frac{\pi_\theta(y_{i,t} \mid x, y_{i,<t})}
{\pi_{\theta_{\mathrm{old}}}(y_{i,t} \mid x, y_{i,<t})},
\qquad
\mathrm{KL}_{i,t}(\theta)
:=
D_{\mathrm{KL}}\!\left(
\pi_\theta(\cdot \mid x,y_{i,<t})
\,\middle\|\,
\pi_{\mathrm{ref}}(\cdot \mid x,y_{i,<t})
\right).
\end{equation} 

%% file: overcoming_learning_cliff.tex
\section{Overcoming Learning Cliffs with Privileged Guidance }
\label{sec:improving_grpo}


Challenging reasoning benchmarks often contain problems that receive almost no verifier reward under the behavior policy $\pi_{\theta_{\mathrm{old}}}$. We call a problem $x$ \emph{hard} if its expected verifier reward under the behavioral policy
\begin{equation}
\label{eq:p_x}
{
p(x)
\;:=\;
\mathbb{E}_{y \sim \pi_{\theta_{\mathrm{old}}}(\cdot \mid x)}[r(x,y)]}
\end{equation}
is close to zero. Hard problems are common in practice, and they expose a fundamental failure mode of GRPO: if every sampled response is incorrect, rewards are uniformly zero and $\sigma_r=0$, so the group-relative advantages $\widehat{A}_i$ vanish. This means the gradient of~\eqref{eq:grpo} is identically zero, and nothing is learned from $x$. We refer to this as the \emph{learning cliff} in RLVR.

\begin{wrapfigure}[14]{r}{0.43\textwidth}
  \centering
  \vspace{-1em}
  \includegraphics[width=0.43\textwidth]{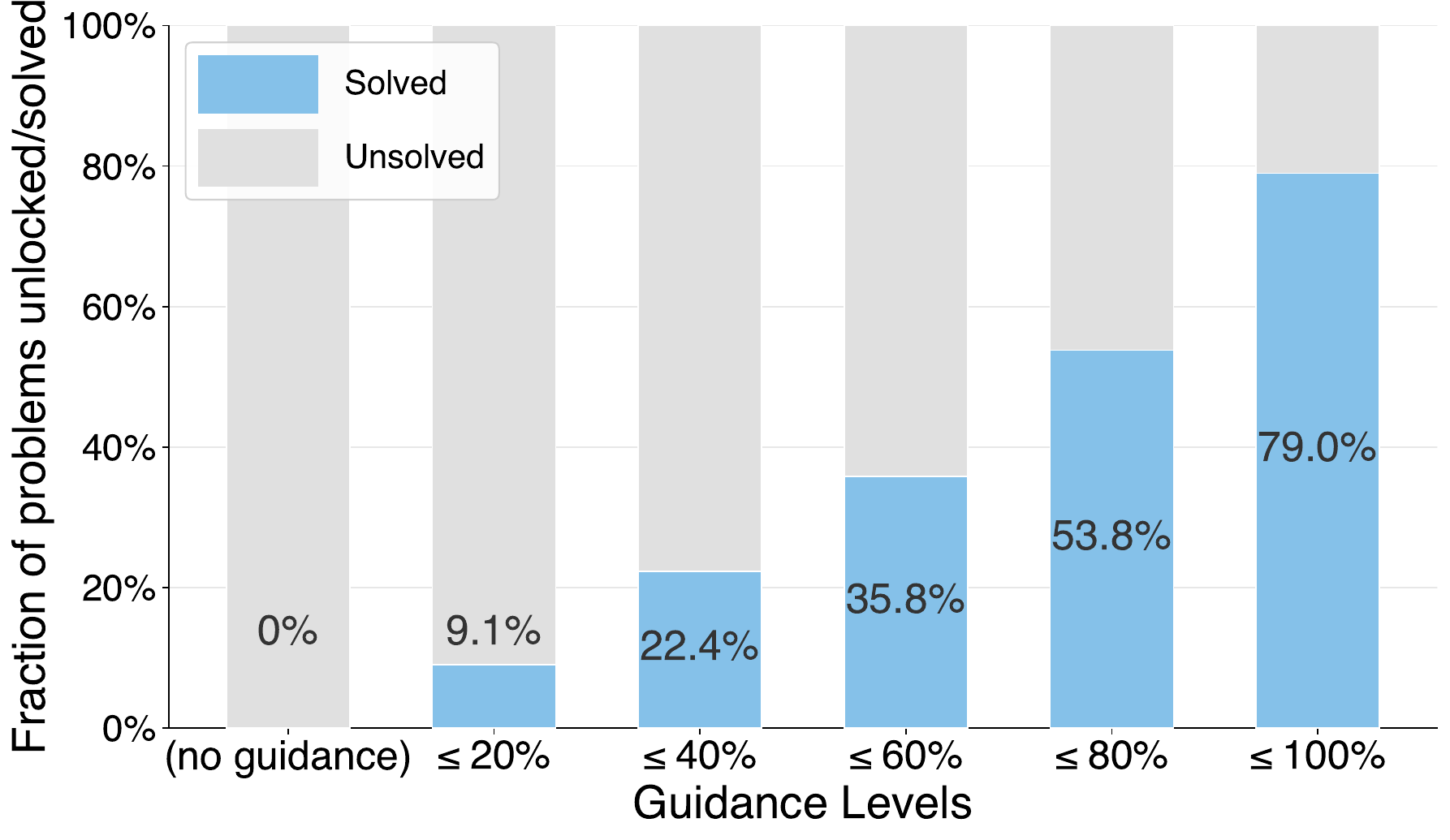}
  \caption{\small{Solution prefixes increase empirical expected verifier reward under $\pi_{\theta_{\mathrm{old}}}$ on hard MATH problems from 0.091 (none) to 0.790 (full prefix)~\citep{hendrycksmath2021}.}}
  \label{fig:guidance_unlocks}
  \vspace{-1em}
\end{wrapfigure}
A direct way to address the learning cliff is to use \emph{privileged guidance}: auxiliary training-time context, unavailable at inference, that raises expected verifier reward on a hard problem $x$. Let $g$ be a guidance strategy that maps each problem to a rollout prompt. It leaves non-hard problems unchanged, $g(x)=x$, and can add a solution prefix or problem-specific hint on hard problems (Appendix~\ref{app:hint_families}). The expected reward under the guided prompt is
\begin{equation}
\label{eq:q_x}
{
q(x)
\;:=\;
\mathbb{E}_{y \sim \pi_{\theta_{\mathrm{old}}}(\cdot \mid g(x))}[r(x,y)],}
\end{equation}
Useful guidance satisfies $q(x) > p(x)$ (see Theorem~\ref{thm:credit_assignment}). Figure~\ref{fig:guidance_unlocks} gives an illustration of this effect for increasing solution prefixes.

One way that recent methods have used privileged guidance is to sample rollouts under $g(x)$ and run GRPO as if $g(x)$ were the true task~\citep{zhang2025scaf, zhangbread, qu2026pope, setlur2026reuse, amani2025rl}. These guided-target methods effectively optimize
\begin{equation}
\label{eq:j_guide}
J^{\mathrm{guide}}(\theta)
\;:=\;
\mathbb{E}_{x \sim \mathcal{D}}\,
\mathbb{E}_{y \sim \pi_\theta(\cdot \mid g(x))}\!
\bigl[r(x, y)\bigr],
\end{equation}
which scores the original problem $x$ but trains the policy under the behavior prompt $g(x)$. The corresponding per-token ratio in~\eqref{eq:grpo_ratio} is
\begin{equation}
\label{eq:trajectory_ratio_others}
\rho^{\mathrm{guide}}_{i,t}(\theta)
\;\coloneqq\;
\frac{\pi_\theta(y_{i,t} \mid g(x), y_{i,<t})}
     {\pi_{\theta_{\mathrm{old}}}(y_{i,t} \mid g(x), y_{i,<t})},
\end{equation}
conditioning on $g(x)$ throughout rather than the original $x$ for hard problems. The objective $J^{\mathrm{guide}}$ is therefore different from the target $J$ because it trains the policy on responses conditioned on $g(x)$, while $J$ evaluates responses conditioned on $x$. When guidance carries significant information, these two can differ substantially. We refer to this mismatch as the \emph{off-context} problem.

Despite optimizing a misaligned objective, guided-target updates can still improve $J$ in practice. This transfer rests on \emph{back-generalization}, a hypothesis coined by \citet{setlur2026reuse} and shared by related guided-target methods~\citep{qu2026pope}: the rollout distributions induced by $g(x)$ and $x$ overlap, so a policy update that improves continuations under the guided prompt also improves them under the nearby unguided prompt. Back-generalization is empirical and not guaranteed in practice: Appendix~\ref{app:toy_example} gives a toy example where optimizing $J^{\mathrm{guide}}$ selects a policy that is useless under the original prompt, and our experiments in Section~\ref{sec:exp_results} suggest the transfer is capacity-dependent, with guided-target baselines degrading below vanilla GRPO at smaller model scales.

%% file: off_context_GRPO.tex
\section{Off-Context GRPO}\label{sec:oc_grpo}


\algoname{} addresses the off-context problem via a minimal modification to the GRPO update rule. Recall from Section~\ref{sec:improving_grpo} that existing methods effectively optimize $J^{\mathrm{guide}}(\theta)$ by sampling rollouts under the guided prompt $g(x)$ instead of the original prompt $x$. Importance sampling provides a principled correction:
\[
\mathbb{E}_{x \sim \mathcal{D},\,y \sim \pi_{\theta_{\mathrm{old}}}(\cdot \mid g(x))}\!\left[
\frac{\pi_\theta(y \mid x)}{\pi_{\theta_{\mathrm{old}}}(y \mid g(x))}\, f(y)
\right]
\;=\;
\mathbb{E}_{x \sim \mathcal{D},\,y \sim \pi_\theta(\cdot \mid x)}[f(y)].
\]
This makes the reweighted guided expectation exactly equivalent to the target expectation. When $g(x)=x$ (i.e., no guidance is applied), the per-token importance ratio reduces to the standard GRPO ratio in~\eqref{eq:grpo_ratio}, and the update is identical to standard GRPO.

In our setting, the sampling distribution is $\pi_{\theta_{\mathrm{old}}}(\cdot \mid g(x))$, the guided context that produced the rollout, and the target distribution is $\pi_\theta(\cdot \mid x)$, the unguided prompt that defines $J(\theta)$ in~\eqref{eq:rl_objective}. The appropriate per-token importance ratio for guided rollouts is therefore
\begin{equation}\label{eq:ic_trajectory_ratio}
\rho^{\mathrm{oc}}_{i,t}(\theta)
\;\coloneqq\;
\frac{\pi_\theta(y_{i,t} \mid x, y_{i,<t})}
     {\pi_{\theta_{\mathrm{old}}}(y_{i,t} \mid g(x), y_{i,<t})}.
\end{equation}

\begin{remark}
\label{rem:masking}
One alternative is to mask the guidance tokens from the loss~\citep{setlur2026reuse}. This restricts which tokens receive gradient, but does not change what they condition on. The per-token ratio remains $\pi_\theta(y_{i,t} \mid g(x), y_{i,<t})\,/\,\pi_{\theta_{\mathrm{old}}}(y_{i,t} \mid g(x), y_{i,<t})$, with $g(x)$ in both numerator and denominator, so the update still targets $J^{\mathrm{guide}}(\theta)$ rather than $J(\theta)$. We verify this failure mode empirically in Section~\ref{sec:main_results}.
\end{remark}

Table~\ref{tab:four_objectives} summarizes the contemporary landscape of approaches along three axes: guided sampling (for overcoming learning cliff), off-context correction, and effective objective. Table~\ref{tab:four_objectives} highlights that \algoname{} is the unique combination that achieves all three desiderata.
\begin{table}[h]
\centering
\small
\setlength{\tabcolsep}{4pt}
\renewcommand{\arraystretch}{1.6}
\renewcommand{\tabularxcolumn}[1]{m{#1}}
\begin{tabularx}{\textwidth}{>{\raggedright\arraybackslash}m{0.42\textwidth}*{3}{>{\centering\arraybackslash}X}}
\toprule
\textbf{Method}
& \makecell[c]{\textbf{Guided}\\\textbf{Sampling}}
& \makecell[c]{\textbf{Off-Context}\\\textbf{Correction}}
& \makecell[c]{\textbf{Effective}\\\textbf{Objective is $J(\theta)$}} \\
\midrule
Vanilla GRPO &
\textcolor{red}{\ding{55}} &
\textcolor{red}{\ding{55}} &
\textcolor{green!60!black}{\ding{51}} \\
Guided, on-context\newline
{\scriptsize\citet{qu2026pope, zhangbread}\newline
\citet{setlur2026reuse, amani2025rl}} &
\textcolor{green!60!black}{\ding{51}} &
\textcolor{red}{\ding{55}} &
\textcolor{red}{\ding{55}}\\
Guided, no correction &
\textcolor{green!60!black}{\ding{51}} &
\textcolor{red}{\ding{55}} &
\textcolor{red}{\ding{55}}\\
\addlinespace[0.35em]
\rowcolor{green!6}
\textbf{OC-GRPO (ours)} &
\textcolor{green!60!black}{\ding{51}} &
\textcolor{green!60!black}{\ding{51}} &
\textcolor{green!60!black}{\ding{51}}\\
\bottomrule
\end{tabularx}
\vspace{0.5em}
\caption{Comparison of GRPO-based methods along three axes: whether
guided sampling is used to address the learning cliff, whether the
off-context shift is corrected in the update rule, and whether the resulting effective objective is aligned to the target objective. \algoname{} is the only method that combines
guided sampling with off-context correction, recovering $J(\theta)$
as the optimization target. Appendix~\ref{app:toy_example} gives a toy example where guided-target optimization selects the opposite policy from $J(\theta)$; Appendix~\ref{app:collapse} demonstrates the instability of uncorrected guided training.}
\label{tab:four_objectives}
\end{table}
Replacing $\rho_{i,t}(\theta)$ in~\eqref{eq:grpo} with $\rho^{\mathrm{oc}}_{i,t}(\theta)$ from~\eqref{eq:ic_trajectory_ratio} gives the \algoname{} surrogate:
\begin{equation}
\label{eq:oc_grpo_loss}
\begin{aligned}
\mathcal{L}_{\algoname}(\theta)
&=
\mathbb{E}_{x \sim \mathcal{D}}
\left[
\frac{1}{G}\sum_{i=1}^{G}
\frac{1}{T_i}\sum_{t=1}^{T_i}
\left\{
\min\!\left(
\rho^{\mathrm{oc}}_{i,t}(\theta)\,\widehat{A}_i,\;
\mathrm{clip}\!\left(\rho^{\mathrm{oc}}_{i,t}(\theta),\, 1{-}\epsilon,\, 1{+}\epsilon\right)\widehat{A}_i
\right)
- \beta\,\mathrm{KL}_{i,t}(\theta)
\right\}
\right].
\end{aligned}
\end{equation}

\section{Theoretical Understanding}
\label{sec:understanding_oc_grpo}

We now show that the importance correction in \eqref{eq:ic_trajectory_ratio} has a more intuitive interpretation than simply to correct the off-context distribution mismatch. In fact, it induces a \emph{provably} well-behaved, behavior-aware credit assignment mechanism, automatically adjusting the gradient contributions based on how much the guidance influenced each trajectory. For simplicity, the analysis in this subsection is response-level, where we treat a sampled completion $y=(y_1,\ldots,y_T)$ as the unit that receives verifier reward and group-relative advantage.

Let $\mathcal{Y}^+=\{y:r(x,y)=1\}$ and $\mathcal{Y}^-=\{y:r(x,y)=0\}$ refer to the responses partitioned based on their correctness. Recall from~\eqref{eq:p_x} and~\eqref{eq:q_x} that, for a hard problem $x$,
\[
    p(x) = \mathbb{E}_{y\sim\pi_{\theta_{\mathrm{old}}}(\cdot|x)}[r(x,y)] \approx 0,
    \qquad
    q(x) = \mathbb{E}_{y\sim\pi_{\theta_{\mathrm{old}}}(\cdot|g(x))}[r(x,y)],
\]
for the expected verifier rewards under the original and guided prompts respectively. For this response-level calculation, set $\theta=\theta_{\mathrm{old}}$. We use the centered advantage
\[
\widehat{A}(x,y)
:=
r(x,y)-q(x)
=
\begin{cases}
1-q(x), & y\in\mathcal{Y}^+,\\
-q(x), & y\in\mathcal{Y}^-,
\end{cases}
\]
without within-group standard-deviation normalization. Define the response-level correction multiplier and aggregate correction factors
\[
\begin{aligned}
\rho_{\mathrm{resp}}^{\mathrm{oc}}(y)
&:=
\frac{\pi_{\theta_{\mathrm{old}}}(y\mid x)}
     {\pi_{\theta_{\mathrm{old}}}(y\mid g(x))},
\quad
\lambda_+(x,g)
:=
\mathbb{E}\!\left[\rho_{\mathrm{resp}}^{\mathrm{oc}}(y)\;\middle|\; y\in\mathcal{Y}^+\right],
\quad
\lambda_-(x,g)
:=
\mathbb{E}\!\left[\rho_{\mathrm{resp}}^{\mathrm{oc}}(y)\;\middle|\; y\in\mathcal{Y}^-\right].
\end{aligned}
\]
The quantities $\lambda_+(x,g)$ and $\lambda_-(x,g)$ are conditional averages of the response-level correction multiplier over successful and failed guided rollouts, respectively. Pointwise, $\rho_{\mathrm{resp}}^{\mathrm{oc}}(y)>1$ means the original prompt assigns higher probability to the same response $y$ than the guided prompt does; $\rho_{\mathrm{resp}}^{\mathrm{oc}}(y)<1$ means the guided prompt assigns higher probability to $y$; and $\rho_{\mathrm{resp}}^{\mathrm{oc}}(y)=1$ means the two probabilities match.
We also use the ``ideal'' \emph{response-level} surrogate, without clipping or
token-level averaging,
\[
\mathcal{L}_{\mathrm{OC}\text{-}\mathrm{GRPO}}^{\mathrm{resp}}(\theta)
:=
\mathbb{E}_{x\sim\mathcal{D}}
\left[
\mathbb{E}_{y\sim\pi_{\theta_{\mathrm{old}}}(\cdot\mid g(x))}
\!\left[
\frac{\pi_\theta(y\mid x)}
     {\pi_{\theta_{\mathrm{old}}}(y\mid g(x))}
\widehat{A}(x,y)
\right]
\right].
\]
At $\theta=\theta_{\mathrm{old}}$, the prefactor in the inner expectation is
$\rho_{\mathrm{resp}}^{\mathrm{oc}}(y)$.
We first isolate the aggregate effect of the correction on successful and failed guided rollouts.
The proofs of Lemma~\ref{lem:lambda_factors} and Theorem~\ref{thm:credit_assignment} are deferred to Appendix~\ref{app:missing_proofs}.
\begin{lemma}[Aggregate correction factors]
\label{lem:lambda_factors}
Consider a hard problem $x$ with $0<p(x)<q(x)<1$, and assume
$\pi_{\theta_{\mathrm{old}}}(y\mid g(x))>0$ whenever
$\pi_{\theta_{\mathrm{old}}}(y\mid x)>0$. Then, with all expectations taken over
guided rollouts $y\sim\pi_{\theta_{\mathrm{old}}}(\cdot\mid g(x))$,
\[
\lambda_+(x,g)=\frac{p(x)}{q(x)}<1,
\qquad
\lambda_-(x,g)=\frac{1-p(x)}{1-q(x)}>1.
\]
\end{lemma}
Lemma~\ref{lem:lambda_factors} gives a direct credit-assignment interpretation, formalized in Theorem~\ref{thm:credit_assignment} below.

\begin{theorem}[\algoname{} Credit Assignment]
\label{thm:credit_assignment}
Under the conditions of Lemma~\ref{lem:lambda_factors}, the following hold:
\begin{enumerate}[label=\textup{(\roman*)}]
    \item \textbf{Systematic credit redistribution.}
    The response-level correction suppresses expected credit on successful guided rollouts,
    \[
    \mathbb{E}\!\left[\rho_{\mathrm{resp}}^{\mathrm{oc}}(y)\widehat{A}(x,y)\;\middle|\; y\in\mathcal{Y}^+\right]
    <
    \mathbb{E}\!\left[\widehat{A}(x,y)\;\middle|\; y\in\mathcal{Y}^+\right],
    \]
    and increases the expected penalty magnitude on failed guided rollouts,
    \[
    \Bigl\lvert\,\mathbb{E}\!\left[\rho_{\mathrm{resp}}^{\mathrm{oc}}(y)\widehat{A}(x,y)\;\middle|\; y\in\mathcal{Y}^-\right]\,\Bigr\rvert
    >
    \Bigl\lvert\,\mathbb{E}\!\left[\widehat{A}(x,y)\;\middle|\; y\in\mathcal{Y}^-\right]\,\Bigr\rvert.
    \]

    \item \textbf{Gradient decomposition.}
    Define
$s_x(y):=\left.\nabla_\theta\log\pi_\theta(y\mid x)\right|_{\theta=\theta_{\mathrm{old}}}$. The gradient of the response-level surrogate satisfies the following decomposition:
    \begin{equation}
    \label{eq:resp_grad_decomp}
    \begin{aligned}
    \left.\nabla_\theta
    \mathcal{L}_{\mathrm{OC}\text{-}\mathrm{GRPO}}^{\mathrm{resp}}(\theta)
    \right|_{\theta=\theta_{\mathrm{old}}}
    &=
    \mathbb{E}_{x\sim\mathcal{D}}
    \Bigg[
    \underbrace{\lambda_+(x,g)(1-q(x))}_{\text{success weight}}
    \cdot q(x)\cdot
    \mathbb{E}_{y\sim\pi_{\theta_{\mathrm{old}}}(\cdot\mid x)}\!\left[
    s_x(y)\;\middle|\; y\in\mathcal{Y}^+
    \right] \\
    &\hspace{4.6em}
    -
    \underbrace{\lambda_-(x,g)q(x)}_{\text{failure weight}}
    \cdot (1-q(x))\cdot
    \mathbb{E}_{y\sim\pi_{\theta_{\mathrm{old}}}(\cdot\mid x)}\!\left[
    s_x(y)\;\middle|\; y\in\mathcal{Y}^-
    \right]
    \Bigg].
    \end{aligned}
    \end{equation}
    For hard prompts with $p(x)\approx0$, standard unguided GRPO has little reward
    variation, while \eqref{eq:resp_grad_decomp} still produces a failure-side update
    whenever $q(x)>0$ and the failure-conditioned score average is nonzero.
\end{enumerate}
\end{theorem}
 
Part~\textup{(ii)} shows where the correction from part~\textup{(i)} enters the gradient. The success branch is scaled by $\lambda_+(x,g)=p(x)/q(x)$, while the failure branch is scaled by $\lambda_-(x,g)=(1-p(x))/(1-q(x))$. The importance ratio therefore acts as a behavior-aware credit assignment mechanism, where the correction multiplier encodes the guidance's influence on each trajectory and automatically adjusts gradient contributions accordingly without any additional design choices. Since $\pi_\theta$ is conditioned only on $x$ throughout (never on $g(x)$), the privileged information $g$ plays no role at inference time, and the training and deployment objectives remain exactly aligned.

The following example illustrates this response-level mechanism for light and large hints.
\begin{example}[Response-level credit assignment]
In practice, the correction credits trajectories that solve a hard problem without leaning too heavily on the guidance. Consider the shared problem $x$, ``Prove $|a\!\cdot\!b| \le \|a\|\,\|b\|$,'' together with three guidance strategies $g_1,g_2,g_3$.
\begin{itemize}
    \item \textbf{Light hint, successful rollout.} The guidance says ``consider $\|a-b\|^2\ge0$'', and the response $y$ expands $(a-b)\!\cdot\!(a-b)$, rearranges, and reaches the conclusion. The model essentially solves the problem on its own, so the guidance barely shifts the likelihood of $y$. The correction multiplier is close to one, retaining full credit.
    \item \textbf{Light hint, failed rollout.} The guidance gives a similarly light hint, but the model goes off-track without using it. Because this failing trajectory does not rely on the guidance, the same ratio is again close to one, so the full penalty is preserved.
    \item \textbf{Large hint, successful rollout.} The guidance gives away the solution, ``$\|a-b\|^2\!\ge\!0 \Rightarrow a\!\cdot\!b \le \tfrac{1}{2}(\|a\|^2{+}\|b\|^2) \Rightarrow$ AM--GM $\Rightarrow$ result'', which the response reproduces verbatim. Here the guidance is almost entirely responsible for $y$, so the correction multiplier is much smaller than one, heavily damping the credit even though the trajectory is correct.
\end{itemize}
The correction preserves the learning signal when the model succeeds or fails on its own merits, and suppresses it when a success is inherited from the guidance.
\end{example}
In Appendix~\ref{app:variance_proofs}, we also analyze the variance of the correction ratio as a function of the guidance length.


%% file: empirical_results.tex
\section{Empirical Results}
\label{sec:main_results}
We evaluate \algoname{} on math reasoning benchmarks to assess whether importance-corrected updates can reliably learn from privileged guidance. 
We describe the experimental setup below and present results in Section~\ref{sec:exp_results}.

\subsection{Implementing \algoname{}}
\label{sec:implementation}

\algoname{} requires three components: (i) identifying which problems need 
guidance and at what level, (ii) constructing the augmented training dataset, 
and (iii) running the importance-corrected update.

\paragraph{Identifying hard problems.} We consider the MATH dataset~\citep{hendrycksmath2021}, which includes full reference solutions. A problem $x$ is hard if $\pi_{\mathrm{ref}}$ is unable to solve it in $64$ attempts. We start from all the hard problems in MATH Levels 3, 4, and 5.

\paragraph{Using cascaded guidance.} Given a hard problem $x \in \mathcal{D}$ with reference solution $z^\star(x) = (z^\star_1, \ldots, z^\star_{T^\star})$, we define $L = 5$ guidance levels by taking solution prefixes at fixed length fractions. Let $h_\ell(x)$ denote the prefix text and $g_\ell(x)$ denote the corresponding rollout prompt:
\begin{equation}\label{eq:prefix_levels}
h_\ell(x) \;:=\; z^\star_{1:\lceil \alpha_\ell \cdot T^\star \rceil}(x),
\qquad
g_\ell(x) \;:=\; (x,h_\ell(x)),
\qquad \alpha_\ell \in \{0.2,\, 0.4,\, 0.6,\, 0.8,\, 1.0\}.
\end{equation}
Levels are strictly nested: $h_\ell(x)$ contains all the information in $h_{\ell-1}(x)$ plus an additional chunk of the reference solution. Alternative guidance families (self-hints, frontier hints) are studied in Appendix~\ref{app:hint_families}, and the same methodology and theoretical results apply to those settings.

\paragraph{Fixed guidance via \algonameFixed{}.} \algonameFixed{} constructs an augmented training dataset $\mathcal{D}_{\mathrm{aug}}$ \emph{before} RLVR training begins, by identifying the minimum guidance level needed for each hard problem to yield at least one correct rollout under the base model $\pi_{\mathrm{ref}}$. For each $x \in \mathcal{D}$, we start from $\ell = 1$, sample 8 rollouts from $\pi_{\mathrm{ref}}(\cdot \mid g_\ell(x))$; if any succeeds, record $\ell^\star(x) = \ell$ and stop; otherwise adjust to $\ell + 1$. The augmented dataset is then
\begin{equation*}\label{eq:augmented_dataset}
\mathcal{D}_{\mathrm{aug}} \;=\; \mathcal{D} \;\cup\;
\bigl\{g_{\ell^\star(x)}(x) : x \in \mathcal{D}\bigr\}.
\end{equation*}
The guidance level $\ell^\star(x)$ is selected once using the base model and fixed throughout training. The policy never re-selects guidance as it evolves.

\paragraph{Adaptive guidance via \algonameAdaptive{}.}
A natural alternative is to adjust guidance adaptively during training, using the current policy $\pi_{\theta_t}$ to re-evaluate which problems need guidance at each step. We refer to this as \algonameAdaptive{} and defer the full algorithm to Appendix~\ref{app:oc_grpo_on}. Adaptive adjustment can in principle provide more targeted guidance as the policy improves, but incurs additional inference cost at every training step. We report results for both variants, with \algonameFixed{} as our recommended default: it is simpler, faster, and easier to reproduce, while remaining empirically competitive.



\subsection{Experimental Setup}
We construct a hard dataset $\mathcal{D}$ from the MATH training
split (Levels 3--5), retaining problems on which all 64 rollouts under
$\pi_{\mathrm{ref}}$ fail; this yields 595 hard problems on which
vanilla GRPO receives near-zero learning signal. We train
Qwen2.5-7B-Instruct \citep{qwen2.5} with LoRA adapters via
veRL~\citep{sheng2024hybridflow} for 4 epochs (16 rollouts/prompt, 3
seeds). We evaluate Pass@1 and Pass@16\footnote{Pass@$k$ is the probability that at least one of $k$ sampled responses solves the problem. We compute it with the standard unbiased estimator~\citep{chen2021codex}: given $n \ge k$ samples per problem of which $c$ are correct, $\mathrm{pass}@k = \mathbb{E}\bigl[1 - \binom{n-c}{k}/\binom{n}{k}\bigr]$; we use $n=16$.} on AIME (1983--2026),
Gaokao2023, and OmniMath. Full details of the setup are in Appendix~\ref{app:experimental_details}.

\paragraph{Baselines.}
We compare \algonameOff{} and \algonameOn{} against four baselines\footnote{We implement each baseline in our setting with the same $\mathcal{D}_{\mathrm{aug}}$ across all methods. We mark these baselines with a * suffix.}.
Each method's optimization objective is given explicitly to clarify
the comparison. \textbf{Vanilla GRPO} optimizes the unguided objective
$J(\theta) = \mathbb{E}_{x \sim \mathcal{D}}
\mathbb{E}_{y \sim \pi_\theta(\cdot\mid x)}[r(x,y)]$ but receives near-
zero signal because every $x$ has $\mathrm{pass}@64 = 0$ under
$\pi_{\mathrm{ref}}$. \textbf{POPE*}~\citep{qu2026pope} (offline) and \textbf{BREAD*}~\citep{zhangbread} (online) sample from $g(x)$
\emph{and} score under $g(x)$ in both numerator and denominator,
optimizing a misaligned \emph{guided} objective
$J^{\mathrm{guide}}(\theta) := \mathbb{E}_{x \sim \mathcal{D}}
\mathbb{E}_{y \sim \pi_\theta(\cdot\mid g(x))}[r(x,y)]$, rather than $J(\theta)$. Appendix~\ref{app:toy_example} gives a toy example where the two objectives prefer opposite policies. \textbf{PrefixRL*}~\citep{setlur2026reuse} injects solution prefixes
in the \emph{output} space rather than the prompt, with gradients
masked on the prefix tokens. Although the mechanism differs, the
optimized objective is similar to those in~\cite{qu2026pope,zhangbread}: the policy is trained
to complete from a forced prefix endpoint, and at inference time the prefix is absent. See
Appendix~\ref{app:experimental_details} for full empirical implementation details.

\subsection{Results}
\label{sec:exp_results}
\begin{table*}[t]
\centering
\scriptsize
\setlength{\tabcolsep}{3.2pt}
\renewcommand{\arraystretch}{1.52}
\renewcommand{\tabularxcolumn}[1]{m{#1}}
\begin{tabularx}{\textwidth}{>{\raggedright\arraybackslash}p{0.17\textwidth}*{2}{>{\centering\arraybackslash}X}@{\hspace{7pt}}*{8}{>{\centering\arraybackslash}X}}
\toprule
\textbf{Method}
& \multicolumn{2}{c}{\textbf{Training}}
& \multicolumn{2}{c}{\makecell{\textbf{AIME}\\\textbf{(1983--2026)}}}
& \multicolumn{2}{c}{\textbf{Gaokao2023}}
& \multicolumn{2}{c}{\textbf{OmniMath}}
& \multicolumn{2}{c}{\textbf{Average}} \\
\cmidrule(r){2-3}\cmidrule(lr){4-5}\cmidrule(lr){6-7}\cmidrule(lr){8-9}\cmidrule(lr){10-11}
& \textbf{P@1} & \textbf{P@16}
& \textbf{P@1} & \textbf{P@16}
& \textbf{P@1} & \textbf{P@16}
& \textbf{P@1} & \textbf{P@16}
& \textbf{P@1} & \textbf{P@16} \\
\midrule

\rowcolor{gray!10}
\multicolumn{11}{l}{\textbf{Baselines}} \\
\addlinespace[0.15em]

Base
& 0.7 & 4.7
& 16.8 & 37.5
& 44.2 & 64.7
& 19.5 & 44.8
& 26.8 & 49.0 \\

Vanilla GRPO~{\tiny\citep{shao2024deepseekmath} (ref.)}
& 1.0$\pm$0.2 & 7.0$\pm$0.6
& 17.5$\pm$0.6 & 38.8$\pm$1.2
& 44.0$\pm$3.2 & 64.2$\pm$2.5
& 22.0$\pm$0.8 & \textbf{45.9$\pm$1.2}
& 27.8 & 49.6 \\

\midrule
\addlinespace[0.20em]
\rowcolor{gray!10}
\multicolumn{11}{c}{\textbf{Before-Training Guidance}} \\
\addlinespace[0.30em]

PrefixRL*~{\tiny\citep{setlur2026reuse}}
& \rescell{4.0$\pm$1.4}{(+300.0\%)} & \rescell{11.8$\pm$0.8}{(+68.6\%)}
& \rescell{18.6$\pm$0.7}{(+6.3\%)} & \rescell{40.8$\pm$0.6}{(+5.1\%)}
& \rescell{43.9$\pm$1.8}{(-0.2\%)} & \rescell{65.2$\pm$1.7}{(+1.5\%)}
& \rescell{22.6$\pm$0.8}{(+2.9\%)} & \rescell{44.9$\pm$0.1}{(-2.1\%)}
& \rescell{28.4}{(+1.9\%)} & \rescell{50.3}{(+1.3\%)} \\

POPE*~{\tiny\citep{qu2026pope}}
& \rescell{\textbf{6.8$\pm$0.4}}{(+580.0\%)} & \rescell{13.2$\pm$0.5}{(+88.6\%)}
& \rescell{18.7$\pm$0.2}{(+6.5\%)} & \rescell{\textbf{41.9$\pm$1.0}}{(+8.0\%)}
& \rescell{48.9$\pm$2.1}{(+11.2\%)} & \rescell{67.2$\pm$1.6}{(+4.6\%)}
& \rescell{23.2$\pm$0.8}{(+5.6\%)} & \rescell{45.0$\pm$0.4}{(-1.8\%)}
& \rescell{30.3}{(+8.7\%)} & \rescell{\textbf{51.4}}{(+3.5\%)} \\

\rowcolor{green!5}
OC-GRPO-Fixed
& \rescell{6.2$\pm$1.4}{(+520.0\%)} & \rescell{\textbf{14.8$\pm$0.3}}{(+111.4\%)}
& \rescell{19.1$\pm$0.2}{(+9.2\%)} & \rescell{41.7$\pm$0.2}{(+7.4\%)}
& \rescell{\textbf{51.6$\pm$2.3}}{(+17.3\%)} & \rescell{\textbf{68.2$\pm$1.0}}{(+6.2\%)}
& \rescell{\textbf{24.3$\pm$0.7}}{(+10.8\%)} & \rescell{44.3$\pm$1.0}{(-3.5\%)}
& \rescell{\textbf{31.7}}{(+13.8\%)} & \rescell{\textbf{51.4}}{(+3.6\%)} \\

\addlinespace[0.35em]
\midrule
\addlinespace[0.20em]
\rowcolor{gray!10}
\multicolumn{11}{c}{\textbf{During-Training Guidance}} \\
\addlinespace[0.30em]

BREAD*~{\tiny\citep{zhangbread}}
& \rescell{5.7$\pm$0.3}{(+470.0\%)} & \rescell{13.8$\pm$1.2}{(+97.1\%)}
& \rescell{\textbf{19.3$\pm$0.3}}{(+10.0\%)} & \rescell{41.6$\pm$0.8}{(+7.1\%)}
& \rescell{46.0$\pm$1.2}{(+4.5\%)} & \rescell{65.1$\pm$0.5}{(+1.3\%)}
& \rescell{23.3$\pm$0.0}{(+6.1\%)} & \rescell{45.5$\pm$0.4}{(-0.8\%)}
& \rescell{29.5}{(+6.1\%)} & \rescell{50.7}{(+2.2\%)} \\

\rowcolor{green!5}
OC-GRPO-Adaptive
& \rescell{2.5$\pm$0.7}{(+150.0\%)} & \rescell{10.0$\pm$0.3}{(+42.9\%)}
& \rescell{18.0$\pm$0.5}{(+2.5\%)} & \rescell{39.8$\pm$0.8}{(+2.6\%)}
& \rescell{50.5$\pm$0.1}{(+14.8\%)} & \rescell{67.5$\pm$0.8}{(+5.1\%)}
& \rescell{23.9$\pm$0.3}{(+8.8\%)} & \rescell{45.6$\pm$0.8}{(-0.5\%)}
& \rescell{30.8}{(+10.7\%)} & \rescell{51.0}{(+2.7\%)} \\
\bottomrule
\end{tabularx}
\caption{Main results on Qwen2.5-7B-Instruct. Each non-base cell reports the seed-averaged (3 seeds) mean$\pm$std on the first line and, on the second line, the relative improvement over Vanilla GRPO, which serves as the reference (ref.) for all comparisons. The Training columns report Pass@1 and Pass@16 on the 595 hard training problems; Average is computed over AIME (1983--2026)~\citep{aime_1983_2024,balunovic_srimatharena_2025}, Gaokao2023~\citep{zhang2023evaluating}, and OmniMath~\citep{gao2024omnimathuniversalolympiadlevel}. In each column, the method with the best mean is bolded; bolding considers the mean alone, with no adjustment for seed variance, and methods with equal means are both bolded.}
\label{tab: main table}
\end{table*}

\begin{table*}[t]
\centering
\scriptsize
\setlength{\tabcolsep}{3.2pt}
\renewcommand{\arraystretch}{1.37}
\renewcommand{\tabularxcolumn}[1]{m{#1}}
\begin{tabularx}{\textwidth}{>{\raggedright\arraybackslash}p{0.17\textwidth}*{2}{>{\centering\arraybackslash}X}@{\hspace{7pt}}*{8}{>{\centering\arraybackslash}X}}
\toprule
\textbf{Method}
& \multicolumn{2}{c}{\textbf{Training}}
& \multicolumn{2}{c}{\makecell{\textbf{AIME}\\\textbf{(1983--2026)}}}
& \multicolumn{2}{c}{\textbf{Gaokao2023}}
& \multicolumn{2}{c}{\textbf{OmniMath}}
& \multicolumn{2}{c}{\textbf{Average}} \\
\cmidrule(r){2-3}\cmidrule(lr){4-5}\cmidrule(lr){6-7}\cmidrule(lr){8-9}\cmidrule(lr){10-11}
& \textbf{P@1} & \textbf{P@16}
& \textbf{P@1} & \textbf{P@16}
& \textbf{P@1} & \textbf{P@16}
& \textbf{P@1} & \textbf{P@16}
& \textbf{P@1} & \textbf{P@16} \\
\midrule

\rowcolor{gray!10}
\multicolumn{11}{l}{\textbf{Baselines}} \\
\addlinespace[0.15em]

Base
& 0.3 & 5.0
& 9.2 & 29.2
& 38.2 & 64.4
& 18.5 & \textbf{44.2}
& 22.0 & 45.9 \\

Vanilla GRPO~{\tiny(ref.)}
& 1.0 & 6.2
& 10.2 & 30.8
& 41.6 & 63.6
& 18.6 & 43.5
& 23.5 & 46.0 \\

\midrule
\addlinespace[0.20em]
\rowcolor{gray!10}
\multicolumn{11}{c}{\textbf{Before-Training Guidance}} \\
\addlinespace[0.30em]

PrefixRL*
& 4.5 & 11.6
& 10.1 {\tiny(-1.0\%)} & \textbf{32.1} {\tiny(+4.2\%)}
& 41.3 {\tiny(-0.6\%)} & 64.4 {\tiny(+1.2\%)}
& 18.8 {\tiny(+1.1\%)} & 41.7 {\tiny(-4.1\%)}
& 23.4 {\tiny(-0.3\%)} & \textbf{46.1} {\tiny(+0.2\%)} \\

POPE*
& \textbf{6.1} & \textbf{14.3}
& 9.8 {\tiny(-4.0\%)} & 30.2 {\tiny(-2.0\%)}
& 43.4 {\tiny(+4.4\%)} & 62.3 {\tiny(-2.0\%)}
& 18.9 {\tiny(+1.6\%)} & 41.3 {\tiny(-5.1\%)}
& 24.0 {\tiny(+2.4\%)} & 44.6 {\tiny(-3.0\%)} \\

\rowcolor{green!5}
OC-GRPO-Fixed
& 1.3 & 9.4
& \textbf{11.0} {\tiny(+7.9\%)} & 30.6 {\tiny(-0.7\%)}
& \textbf{45.7} {\tiny(+10.0\%)} & 64.7 {\tiny(+1.6\%)}
& 18.8 {\tiny(+1.1\%)} & 42.0 {\tiny(-3.4\%)}
& \textbf{25.2} {\tiny(+7.2\%)} & 45.8 {\tiny(-0.4\%)} \\

\addlinespace[0.35em]
\midrule
\addlinespace[0.20em]
\rowcolor{gray!10}
\multicolumn{11}{c}{\textbf{During-Training Guidance}} \\
\addlinespace[0.30em]

BREAD*
& 4.7 & 13.3
& 9.1 {\tiny(-10.9\%)} & 31.3 {\tiny(+1.6\%)}
& 40.0 {\tiny(-3.8\%)} & 62.1 {\tiny(-2.4\%)}
& 18.9 {\tiny(+1.6\%)} & 42.6 {\tiny(-2.1\%)}
& 22.7 {\tiny(-3.4\%)} & 45.3 {\tiny(-1.4\%)} \\

\rowcolor{green!5}
OC-GRPO-Adaptive
& 1.7 & 6.1
& 10.7 {\tiny(+5.0\%)} & 29.3 {\tiny(-4.9\%)}
& 43.4 {\tiny(+4.4\%)} & \textbf{65.2} {\tiny(+2.4\%)}
& \textbf{19.2} {\tiny(+3.2\%)} & 42.3 {\tiny(-2.8\%)}
& 24.4 {\tiny(+4.1\%)} & 45.6 {\tiny(-0.8\%)} \\
\bottomrule
\end{tabularx}
\caption{Results with Qwen2.5-3B-Instruct (single seed). The empirical setting is identical to Table~\ref{tab: main table}.}

\label{tab:main_results_qwen3b_n16_epoch4_subset}
\end{table*}
\begin{table*}[t]
\centering
\scriptsize
\setlength{\tabcolsep}{3.2pt}
\renewcommand{\arraystretch}{1.37}
\renewcommand{\tabularxcolumn}[1]{m{#1}}
\begin{tabularx}{\textwidth}{>{\raggedright\arraybackslash}p{0.17\textwidth}*{2}{>{\centering\arraybackslash}X}@{\hspace{7pt}}*{8}{>{\centering\arraybackslash}X}}
\toprule
\textbf{Method}
& \multicolumn{2}{c}{\textbf{Training}}
& \multicolumn{2}{c}{\makecell{\textbf{AIME}\\\textbf{(1983--2026)}}}
& \multicolumn{2}{c}{\textbf{Gaokao2023}}
& \multicolumn{2}{c}{\textbf{OmniMath}}
& \multicolumn{2}{c}{\textbf{Average}} \\
\cmidrule(r){2-3}\cmidrule(lr){4-5}\cmidrule(lr){6-7}\cmidrule(lr){8-9}\cmidrule(lr){10-11}
& \textbf{P@1} & \textbf{P@16}
& \textbf{P@1} & \textbf{P@16}
& \textbf{P@1} & \textbf{P@16}
& \textbf{P@1} & \textbf{P@16}
& \textbf{P@1} & \textbf{P@16} \\
\midrule

\rowcolor{gray!10}
\multicolumn{11}{l}{\textbf{Baselines}} \\
\addlinespace[0.15em]

Base
& 1.2 & 7.3
& 3.8 & 18.7
& 34.8 & 56.9
& 13.0 & 37.6
& 17.2 & 37.7 \\

Vanilla GRPO~{\tiny(ref.)}
& 1.8 & 7.4
& 4.6 & 21.2
& 33.5 & 59.0
& 14.7 & 36.9
& 17.6 & 39.0 \\

\midrule
\addlinespace[0.20em]
\rowcolor{gray!10}
\multicolumn{11}{c}{\textbf{Before-Training Guidance}} \\
\addlinespace[0.30em]

PrefixRL*
& 2.0 & 11.3
& 3.9 {\tiny(-15.2\%)} & 20.7 {\tiny(-2.4\%)}
& 31.2 {\tiny(-7.0\%)} & 57.7 {\tiny(-2.2\%)}
& 14.0 {\tiny(-4.8\%)} & 37.7 {\tiny(+2.2\%)}
& 16.4 {\tiny(-7.2\%)} & 38.7 {\tiny(-0.9\%)} \\

POPE*
& \textbf{2.9} & 10.9
& 4.1 {\tiny(-10.9\%)} & 21.1 {\tiny(-0.5\%)}
& 33.8 {\tiny(+0.8\%)} & 59.7 {\tiny(+1.3\%)}
& 15.8 {\tiny(+7.5\%)} & 38.1 {\tiny(+3.3\%)}
& 17.9 {\tiny(+1.7\%)} & 39.6 {\tiny(+1.5\%)} \\

\rowcolor{green!5}
OC-GRPO-Fixed
& 2.5 & 10.4
& 4.2 {\tiny(-8.7\%)} & 19.5 {\tiny(-8.1\%)}
& \textbf{37.7} {\tiny(+12.4\%)} & 59.7 {\tiny(+1.3\%)}
& \textbf{16.3} {\tiny(+10.9\%)} & 36.4 {\tiny(-1.4\%)}
& \textbf{19.4} {\tiny(+10.2\%)} & 38.5 {\tiny(-1.4\%)} \\

\addlinespace[0.35em]
\midrule
\addlinespace[0.20em]
\rowcolor{gray!10}
\multicolumn{11}{c}{\textbf{During-Training Guidance}} \\
\addlinespace[0.30em]

BREAD*
& 2.4 & \textbf{11.8}
& 4.6 {\tiny(+0.0\%)} & 19.3 {\tiny(-9.0\%)}
& 33.2 {\tiny(-0.8\%)} & 60.0 {\tiny(+1.8\%)}
& 14.6 {\tiny(-0.7\%)} & \textbf{38.2} {\tiny(+3.5\%)}
& 17.5 {\tiny(-0.7\%)} & 39.2 {\tiny(+0.4\%)} \\

\rowcolor{green!5}
OC-GRPO-Adaptive
& 0.5 & 8.4
& \textbf{4.8} {\tiny(+4.3\%)} & \textbf{21.3} {\tiny(+0.5\%)}
& 35.8 {\tiny(+7.0\%)} & \textbf{62.9} {\tiny(+6.6\%)}
& 16.1 {\tiny(+9.5\%)} & \textbf{38.2} {\tiny(+3.5\%)}
& 18.9 {\tiny(+7.2\%)} & \textbf{40.8} {\tiny(+4.6\%)} \\
\bottomrule
\end{tabularx}
\caption{Results with Qwen2.5-1.5B-Instruct (single seed). The empirical setting is identical to Table~\ref{tab: main table}.}
\label{tab:main_results_qwen15b_n16_epoch4_subset}
\end{table*}
Table~\ref{tab: main table} reports the main results on Qwen2.5-7B-Instruct. \algonameOff{} achieves the best average Pass@1 of $31.7$, a $13.8\%$ relative gain over vanilla GRPO and $1.4$ absolute points ahead of the strongest guided baseline POPE* ($+8.7\%$). \algonameOn{} is competitive ($+10.7\%$) but trails its fixed counterpart with additional inference cost per training step. Across individual benchmarks, \algonameOff{} leads on Gaokao2023 and OmniMath while remaining competitive on AIME.

Tables~\ref{tab:main_results_qwen3b_n16_epoch4_subset} and~\ref{tab:main_results_qwen15b_n16_epoch4_subset} report results at 3B and 1.5B scale with identical hyperparameters. Two patterns emerge consistently. First, \algoname{} remains the strongest method at every scale, with \algonameOff{} achieving $+7.2\%$ at 3B and $+10.2\%$ at 1.5B over vanilla GRPO. Second, guided baselines that optimize a misaligned objective~\citep{qu2026pope, zhangbread, setlur2026reuse} degrade below vanilla GRPO at smaller scales: PrefixRL* drops $7.2\%$ at 1.5B and BREAD* drops $3.4\%$ at 3B while \algoname{} retains consistent gains. This suggests that back-generalization is capacity-dependent: larger models may absorb the objective mismatch, but smaller models cannot, making the importance correction increasingly critical at reduced scale.\looseness-1

\FloatBarrier

%% file: conclusion.tex
\section{Conclusion}
We identified the off-context problem in guided RLVR, where existing methods use privileged guidance to restore learning signal on hard problems but inadvertently optimize a different objective from the one used at deployment. \algoname{} corrects for this via a minimal per-token importance reweighting, so training can sample from guided prompts while still estimating the unguided objective. The same correction also gives behavior-aware credit assignment: it discounts trajectories that leaned heavily on guidance and preserves the penalty for failures that remain under guidance. Gains over guided baselines are consistent across model scales and grow at smaller sizes, where the objective mismatch is hardest to absorb.

A natural extension is multi-turn or agentic RLVR, where privileged context can take the form of tool-call results, intermediate environment states, or oracle subgoals accessible during training but absent at deployment. The off-context problem appears whenever the training-time context differs from the deployment-time context, and the per-token importance correction in \algoname{} generalizes directly. The broader lesson is that privileged information can guide exploration, but the update still has to account for the distribution it came from.

%% file: appendix.tex
\newpage
\appendix

\noindent{\textbf{Limitations.}} OC-GRPO is evaluated at academic scale, Qwen2.5 models up to 7B with LoRA adapters on a small A100 cluster, and behavior at frontier scale  may differ. Our experiments target mathematical reasoning with verifiable answer rewards, and the gains may not transfer uniformly to other RLVR settings such as code generation, agentic tool-use, or multimodal reasoning. Finally, while the off-context framework is mechanism-agnostic and we explore hint-style variants in Appendix~\ref{app:hint_families}, our main results rely on ground-truth solution prefixes as the privileged signal; behavior under partial, noisy, or adversarial guidance is left for future work.

\section{Missing Proofs}
\label{app:missing_proofs}

\subsection{\texorpdfstring{Proof of Lemma~\ref{lem:lambda_factors}}{Proof of Lemma}}
\label{app:proof_lambda_factors}

\begin{proof}
Write
$P(y)=\pi_{\theta_{\mathrm{old}}}(y\mid x)$ and
$Q(y)=\pi_{\theta_{\mathrm{old}}}(y\mid g(x))$. The support assumption lets us
restrict sums to responses with positive guided probability without dropping any
response with positive original-prompt probability. Let
$\mathcal{Y}^+_g=\{y\in\mathcal{Y}^+:Q(y)>0\}$. Then
\[
\begin{aligned}
\lambda_+(x,g)
&=
\mathbb{E}\!\left[\rho_{\mathrm{resp}}^{\mathrm{oc}}(y)\;\middle|\; y\in\mathcal{Y}^+\right] \\
&=
\sum_{y\in\mathcal{Y}^+_g}
\frac{Q(y)}{\sum_{z\in\mathcal{Y}^+}Q(z)}
\frac{P(y)}{Q(y)} \\
&=
\frac{\sum_{y\in\mathcal{Y}^+_g}P(y)}
     {\sum_{z\in\mathcal{Y}^+}Q(z)}
=
\frac{\sum_{y\in\mathcal{Y}^+}P(y)}
     {\sum_{z\in\mathcal{Y}^+}Q(z)} \\
&=
\frac{\Pr_{y\sim\pi_{\theta_{\mathrm{old}}}(\cdot\mid x)}(r(x,y)=1)}
     {\Pr_{y\sim\pi_{\theta_{\mathrm{old}}}(\cdot\mid g(x))}(r(x,y)=1)}
=
\frac{p(x)}{q(x)}.
\end{aligned}
\]
Since $p(x)<q(x)$, this ratio is less than one. The same calculation on
$\mathcal{Y}^-$ gives
\[
\lambda_-(x,g)
=
\frac{\sum_{y\in\mathcal{Y}^-}P(y)}
     {\sum_{y\in\mathcal{Y}^-}Q(y)}
=
\frac{1-p(x)}{1-q(x)}
>1,
\]
where the strict inequality again follows from $p(x)<q(x)$.
\end{proof}

\subsection{\texorpdfstring{Proof of Theorem~\ref{thm:credit_assignment}}{Proof of Theorem}}
\label{app:proof_credit_assignment}

\begin{proof}
We prove the two parts in order. By Lemma~\ref{lem:lambda_factors},
$\lambda_+(x,g)<1<\lambda_-(x,g)$. On
$\mathcal{Y}^+$, the centered advantage is constant and equal to $1-q(x)$, so
$\mathbb{E}[\widehat{A}(x,y)\mid y\in\mathcal{Y}^+]=1-q(x)>0$. Hence
\[
\begin{aligned}
\mathbb{E}\!\left[\rho_{\mathrm{resp}}^{\mathrm{oc}}(y)\widehat{A}(x,y)\;\middle|\; y\in\mathcal{Y}^+\right]
&=
(1-q(x))
\mathbb{E}\!\left[\rho_{\mathrm{resp}}^{\mathrm{oc}}(y)\;\middle|\; y\in\mathcal{Y}^+\right] \\
&=
\lambda_+(x,g)
\mathbb{E}\!\left[\widehat{A}(x,y)\;\middle|\; y\in\mathcal{Y}^+\right] \\
&<
\mathbb{E}\!\left[\widehat{A}(x,y)\;\middle|\; y\in\mathcal{Y}^+\right].
\end{aligned}
\]
On $\mathcal{Y}^-$, the same calculation uses
$\widehat{A}(x,y)=-q(x)$ and $\lambda_-(x,g)>1$, which gives
\[
\begin{aligned}
\Bigl\lvert\,\mathbb{E}\!\left[\rho_{\mathrm{resp}}^{\mathrm{oc}}(y)\widehat{A}(x,y)\;\middle|\; y\in\mathcal{Y}^-\right]\,\Bigr\rvert
&=
\lambda_-(x,g)
\Bigl\lvert\,\mathbb{E}\!\left[\widehat{A}(x,y)\;\middle|\; y\in\mathcal{Y}^-\right]\,\Bigr\rvert \\
&>
\Bigl\lvert\,\mathbb{E}\!\left[\widehat{A}(x,y)\;\middle|\; y\in\mathcal{Y}^-\right]\,\Bigr\rvert.
\end{aligned}
\]
For the second part, fix $x$ and write
$P(y)=\pi_{\theta_{\mathrm{old}}}(y\mid x)$ and
$Q(y)=\pi_{\theta_{\mathrm{old}}}(y\mid g(x))$. Differentiating
$\mathcal{L}_{\mathrm{OC}\text{-}\mathrm{GRPO}}^{\mathrm{resp}}$ and evaluating
at $\theta=\theta_{\mathrm{old}}$ gives
\[
\left.
\nabla_\theta
\mathcal{L}_{\mathrm{OC}\text{-}\mathrm{GRPO}}^{\mathrm{resp}}(\theta)
\right|_{\theta=\theta_{\mathrm{old}}}
=
\mathbb{E}_{x\sim\mathcal{D}}
\left[
\mathbb{E}_{y\sim Q}\!\left[
\rho_{\mathrm{resp}}^{\mathrm{oc}}(y)\widehat{A}(x,y)s_x(y)
\right]
\right].
\]
Because $\widehat{A}(x,y)=1-q(x)$ on $\mathcal{Y}^+$ and
$\widehat{A}(x,y)=-q(x)$ on $\mathcal{Y}^-$, the law of total expectation gives
\[
\begin{aligned}
\mathbb{E}_{y\sim Q}\!\left[
\rho_{\mathrm{resp}}^{\mathrm{oc}}(y)\widehat{A}(x,y)s_x(y)
\right]
&=
q(x)(1-q(x))
\mathbb{E}_{y\sim Q}\!\left[
\rho_{\mathrm{resp}}^{\mathrm{oc}}(y)s_x(y)
\;\middle|\; y\in\mathcal{Y}^+
\right] \\
&\quad
-q(x)(1-q(x))
\mathbb{E}_{y\sim Q}\!\left[
\rho_{\mathrm{resp}}^{\mathrm{oc}}(y)s_x(y)
\;\middle|\; y\in\mathcal{Y}^-
\right].
\end{aligned}
\]
For any $S\in\{\mathcal{Y}^+,\mathcal{Y}^-\}$, the correction changes the
conditional law from $Q$ to $P$:
\[
\begin{aligned}
\mathbb{E}_{y\sim Q}\!\left[
\rho_{\mathrm{resp}}^{\mathrm{oc}}(y)s_x(y)
\;\middle|\; y\in S
\right]
&=
\sum_{y\in S}
\frac{Q(y)}
     {\sum_{z\in S}Q(z)}
\frac{P(y)}{Q(y)}
s_x(y) \\
&=
\frac{\sum_{z\in S}P(z)}
     {\sum_{z\in S}Q(z)}
\mathbb{E}_{y\sim P}\!\left[
s_x(y)\;\middle|\; y\in S
\right].
\end{aligned}
\]
Taking $S=\mathcal{Y}^+$ gives the factor
$p(x)/q(x)=\lambda_+(x,g)$, and taking $S=\mathcal{Y}^-$ gives the factor
$(1-p(x))/(1-q(x))=\lambda_-(x,g)$. Substitution gives
\eqref{eq:resp_grad_decomp}.
\end{proof}

\section{Extended Related Works}
\label{app:related-work}

\subsection{Self-Distillation in RLVR}
\label{app:self-distillation}

A complementary line of work uses the model itself, conditioned on richer information, as a self-teacher. \textbf{OPSD}~\citep{zhao2026self} instantiates a single LLM as both privileged teacher (conditioned on ground-truth answers) and unprivileged student (conditioned only on the problem), training via per-token KL/JSD divergence on student rollouts. \textbf{SDPO}~\citep{hubotter2026reinforcement} extends this to rich-feedback settings such as code execution traces, treating the feedback-conditioned policy as a self-teacher. \textbf{GATES}~\citep{stein2026gates} applies a similar idea with consensus gating to filter unreliable supervision in document-grounded question answering. Conceptual ancestors include \textbf{STaR}~\citep{zelikman2022star}, which bootstraps rationales conditioned on the correct answer, and on-policy distillation~\citep{agarwal2024policy}.

OPSD and SDPO share with OC-GRPO the use of training-time-only privileged context, but the methods differ in three essential ways. \emph{First, they are distillation methods, not RL methods}: training minimizes a divergence to the teacher's logit distribution rather than a verifier reward, and correctness is never directly consulted during the update. \emph{Second, the two paradigms inject privileged context at different points in the pipeline}: OPSD and SDPO sample rollouts from the unguided student and use privileged information only to shape the teacher's per-token distribution, whereas OC-GRPO uses the privileged context to drive the rollout distribution itself and reserves the unguided context for the importance correction. SDPO is flexible in the kinds of feedback the self-teacher can consume; neither method, however, specifically targets the learning-cliff regime, and without further algorithmic modifications both rely in this regime on teacher signal computed over student-sampled traces that are typically incorrect, whereas OC-GRPO's guided rollouts more frequently produce reward-bearing trajectories. \emph{Third, OC-GRPO is provably unbiased for the original unguided RLVR objective $J(\theta)$}, while OPSD and SDPO optimize a distillation divergence that is structurally distinct from $J(\theta)$. The two paradigms are complementary rather than competing, distillation provides dense token-level supervision when the teacher distribution is reliable, RL provides reward-anchored supervision when it is not.

\subsection{Hybrid SFT and RL}
\label{app:hybrid-sft-rl}

\textbf{CHORD}~\citep{MIXCHORD} reframes SFT as a dynamically weighted auxiliary objective inside on-policy RL, with a global coefficient and a token-wise weighting function balancing imitation against exploration. \textbf{Prefix-RFT}~\citep{huang2025blending} blends supervised demonstrations with RFT by sampling prefixes from offline traces as starting points for on-policy continuation, scheduled via a cosine decay over prefix length. \textbf{ReST-EM}~\citep{singh2023beyond} iterates between rejection-sampling rollouts and SFT on filtered successes, viewing the procedure as expectation-maximization. These approaches harmonize the two losses; OC-GRPO instead harmonizes the two distributions: guided rollouts provide the exploration that SFT-RL hybrids would obtain through expert demonstrations, while importance correction plays the role that the SFT loss plays in those methods, namely keeping training aligned with the original deployment objective.

\subsection{Self-Curriculum and Adaptive Difficulty}
\label{app:self-curriculum}

A separate body of work treats problem difficulty as a curriculum variable controlled across training. Methods in this family, including \textbf{SEC}~\citep{chen2025self}, \textbf{VCRL}~\citep{jiang2025vcrl}, \textbf{CLPO}~\citep{zhang2025clpo}, \textbf{AdaCuRL}~\citep{li2026adacurl}, \textbf{DUMP}~\citep{wang2025dump}, and \textbf{DARS}~\citep{yang2025depth}, adjust which problems the policy trains on or how much rollout budget each receives, often using rollout reward variance, advantage magnitude, or learning progress as control signals. We list these works for completeness; OC-GRPO is not a curriculum method and operates in a different regime. Curriculum approaches assume that the policy can learn from at least some problems on its own and select among them; we target the regime where every remaining hard problem has $\textsc{pass@}k \approx 0$ without guidance, so privileged guidance is the lever that makes such problems trainable in the first place, before any curriculum question arises.

\subsection{Sparse-Reward RL: Experience Reuse and Hindsight Relabeling}
\label{app:hindsight}

\textbf{Hindsight experience replay (HER)}~\citep{andrychowicz2017hindsight} is the canonical sparse-reward technique: failed trajectories are relabeled with the goals they actually achieved, converting zero-reward experience into useful learning signal. Recent work lifts this idea to language settings: \textbf{AgentHER}~\citep{ding2026agenther} relabels failed multi-step agent trajectories with hindsight goals via a four-stage pipeline, and \citet{hu2025sample} extend experience replay to LM agents through hindsight trajectory rewriting. \textbf{ExGRPO}~\citep{zhan2025exgrpo} maintains a replay buffer of partially correct rollouts organized by correctness, prioritizes valuable experiences via correctness and entropy signals, and blends replayed and on-policy data through a mixed-policy objective. HER and OC-GRPO are mirror images: HER changes the goal post-hoc to match the trajectory, while OC-GRPO changes the prompt pre-hoc to make the trajectory reachable, then importance-corrects the gradient back to the original goal. ExGRPO and replay methods are orthogonal to OC-GRPO, they reuse trajectories the policy did sample, rather than altering the sampling distribution to obtain new ones.

\section{\algonameAdaptive{}: Adaptive Guidance Adjustment}
\label{app:oc_grpo_on}

\paragraph{\algonameAdaptive{} (adaptive augmentation).}
At each training step $t$, given a batch $\mathcal{B} \subset \mathcal{D}$:
\begin{enumerate}[leftmargin=1.8em, itemsep=2pt, topsep=4pt]
    \item \textbf{Try unguided first.} For each $x \in \mathcal{B}$, sample $N$ rollouts from $\pi_{\theta_t}(\cdot \mid x)$. If at least one is correct, keep the unguided group for this problem.
    \item \textbf{Adjust guidance if all rollouts fail.} If all $N$ unguided rollouts fail, form the guided prompt $g_1(x)$ and sample $N$ new rollouts from $\pi_{\theta_t}(\cdot \mid g_1(x))$. Continue adjusting through levels $\ell = 2, 3, 4, 5$ until at least one rollout is correct, and use that group for the update.
    \item \textbf{Proceed with the policy update.} The batch now contains a mix of unguided groups (for easier problems at the current policy) and guided groups (for problems that are currently too hard).
\end{enumerate}

\algonameAdaptive{} differs from \algonameFixed{} in how the guidance level
is selected: the fixed variant fixes $\ell^\star(x)$ once using the
base model $\pi_{\mathrm{ref}}$, while the adaptive variant re-selects guidance at
every training step with respect to the current policy $\pi_{\theta_t}$. Theorem~\ref{thm:credit_assignment} applies
whenever the behavior context used to sample the update group is fixed before that group is sampled. Adaptive adjustment selects a behavior prompt $b\in\{x,g_\ell(x)\}$, where $\ell$ is the guidance level that first produces a correct rollout. In the implementation we use the realized behavior prompt in the importance ratio denominator, which corresponds to plugging in the selected behavior law $\pi_{\theta_t}(\cdot \mid b)$. \algonameAdaptive{} pays an inference cost at every training step: when the unguided rollouts fail, the model must re-sample $N$ rollouts for each guidance level, potentially up to $L$ times per problem. In contrast, \algonameFixed{} pays this cost once with $\pi_{\mathrm{ref}}$ before training, and $\mathcal D_{\mathrm{aug}}$ can be reused across runs. Empirically, \algonameFixed{} is competitive with or stronger than \algonameAdaptive{} (Table~\ref{tab: main table}), so we recommend the fixed variant as the default.

\section{Collapse of Off-Context GRPO without IS Correction}
\label{app:collapse}
As discussed in the main text, naively using guided rollouts
$y_i \sim \pi_{\theta_{\mathrm{old}}}(\cdot \mid g(x))$ while updating as if
they came from the original prompt $x$ gives a biased gradient estimator. We
call this the ``Masked no-IC'' update. Useful guidance
changes the rollout distribution, and no-IC credits guided trajectories as
though they were unguided samples from the target objective. During training,
this repeated mismatch can accumulate, inflating the gradient norm early in
training and eventually causing reward collapse, as shown in
Figure~\ref{fig:collapse_without_is}.

\begin{figure}[h]
    \centering
\includegraphics[width=.85\textwidth]{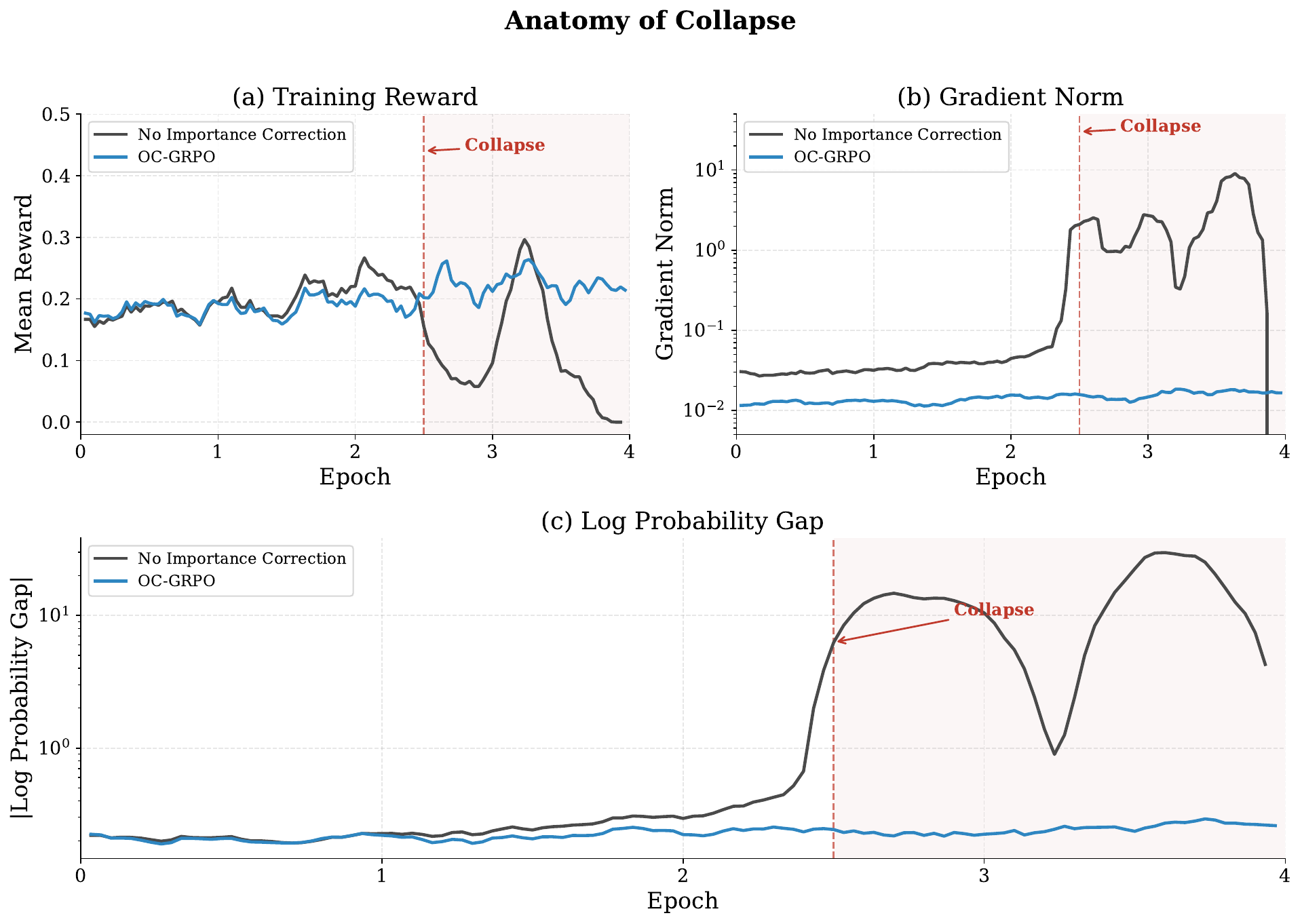}
\caption{\textbf{Importance correction prevents collapse.} Without the importance ratio correction, guided rollouts are scored as if they came from the unguided context, producing growing objective mismatch. This leads to reward collapse, gradient norm spikes, and a widening log-probability gap. \algoname{} keeps the update aligned with the unguided objective.}
\label{fig:collapse_without_is}
\end{figure}

The source of the bias is the off-context mismatch itself. If guidance makes a
correct response much more likely under $g(x)$ than under $x$, then no-IC assigns
too much credit to the unguided policy for a trajectory produced by the guided
rollout distribution. Masking the guidance tokens does not fix this, since the
generated response still came from the guided rollout distribution. \algoname{}
removes this mismatch by weighting each
trajectory by the ratio between the original-prompt probability and the
guided-prompt probability, so the corrected estimator targets $J(\theta)$ by
Theorem~\ref{thm:oc_grpo_unbiased}.

\section{Additional Theoretical Results}

\subsection{Response-Level Unbiasedness}
\label{app:unbiased_proof}

The exact unbiasedness statement is a response-level change of measure. It is separate from the credit-assignment identities in Theorem~\ref{thm:credit_assignment}.

\begin{theorem}[Response-level unbiasedness of the importance correction]
\label{thm:oc_grpo_unbiased}
Fix a problem $x$ and assume the support condition
$\pi_{\theta_{\mathrm{old}}}(y\mid g(x))>0$ whenever
$\pi_\theta(y\mid x)r(x,y)>0$. For a guided rollout
$y\sim\pi_{\theta_{\mathrm{old}}}(\cdot\mid g(x))$, define the response-level estimator
\[
\widehat{G}_{\mathrm{oc}}(\theta;x,y)
:=
\frac{\pi_\theta(y\mid x)}
     {\pi_{\theta_{\mathrm{old}}}(y\mid g(x))}
r(x,y)\nabla_\theta\log\pi_\theta(y\mid x).
\]
Then
\[
\mathbb{E}_{y\sim\pi_{\theta_{\mathrm{old}}}(\cdot\mid g(x))}
\bigl[\widehat{G}_{\mathrm{oc}}(\theta;x,y)\bigr]
=
\mathbb{E}_{y\sim\pi_\theta(\cdot\mid x)}
\bigl[r(x,y)\nabla_\theta\log\pi_\theta(y\mid x)\bigr]
=
\nabla_\theta J(\theta;x).
\]
\end{theorem}

\begin{proof}
By direct substitution,
\[
\begin{aligned}
\mathbb{E}_{y\sim\pi_{\theta_{\mathrm{old}}}(\cdot\mid g(x))}
\bigl[\widehat{G}_{\mathrm{oc}}(\theta;x,y)\bigr]
&=
\sum_y
\pi_{\theta_{\mathrm{old}}}(y\mid g(x))
\frac{\pi_\theta(y\mid x)}
     {\pi_{\theta_{\mathrm{old}}}(y\mid g(x))}
r(x,y)\nabla_\theta\log\pi_\theta(y\mid x)\\
&=
\sum_y
\pi_\theta(y\mid x)
r(x,y)\nabla_\theta\log\pi_\theta(y\mid x)\\
&=
\nabla_\theta
\sum_y \pi_\theta(y\mid x)r(x,y)
=
\nabla_\theta J(\theta;x).
\end{aligned}
\]
The support condition makes the ratio well-defined on every reward-bearing response that can contribute to the target expectation.
\end{proof}

\subsection{Bounded Off-Context Correction Term}
\label{app:variance_proofs}

A practical concern with importance sampling is variance inflation, since large IS weights can destabilize training. The \algoname{}-specific component of the importance ratio is bounded with controlled variance. Decompose the per-token importance ratio as
\begin{equation}
\label{eq:app_variance_ic_decomposition}
\rho^{\mathrm{oc}}_{i,t}(\theta) =
\underbrace{\frac{\pi_\theta(y_{i,t} \mid x, y_{i,<t})}{
\pi_{\theta_{\mathrm{old}}}(y_{i,t} \mid x, y_{i,<t})}}_{\rho_{i,t}(\theta)\
\text{(on-policy drift)}}
\;\cdot\;
\underbrace{\frac{\pi_{\theta_{\mathrm{old}}}(y_{i,t} \mid x, y_{i,<t})}{
\pi_{\theta_{\mathrm{old}}}(y_{i,t} \mid g(x), y_{i,<t})}}_{\gamma_{i,t}\
\text{(off-context correction)}}.
\end{equation}
The factor $\rho_{i,t}(\theta)$ is the standard GRPO drift ratio, controlled by clipping in $[1{-}\epsilon, 1{+}\epsilon]$. The \algoname{}-specific factor is $\gamma_{i,t}$, the context correction.

We use the following regularity assumption.
\begin{assumption}[Per-guidance-token regularity]
\label{ass:eta}
There exists $\eta > 0$ such that for every guidance token index
$k \in \{1, \ldots, n\}$ ($g=u_{1:n}$),  the inequality
\[
\left|\log\frac{
\pi_{\theta_{\mathrm{old}}}(y_{i,t} \mid x, u_{1:k-1}, y_{i,<t})}{
\pi_{\theta_{\mathrm{old}}}(y_{i,t} \mid x, u_{1:k}, y_{i,<t})}\right|
\;\leq\; \eta
\]
holds almost surely over the realized token $y_{i,t} \sim Q_{i,t}$
and the rollout history $y_{i,<t}$.
\end{assumption}
Assumption~\ref{ass:eta} states that adding any single guidance token
shifts the old policy's next-token log-probability by at most $\eta$
in either direction. For a well-calibrated language model on
natural-language input, this is small: each additional context token
refines but does not radically reshape the next-token distribution.

\begin{theorem}[Bounded off-context correction term]
\label{thm:lambda_bounded}
Under Assumption~\ref{ass:eta}, for every $(i,t)$, every history $y_{i,<t}$, and every realized token $y_{i,t}$\textup{:}
\begin{enumerate}
    \item \emph{(Deterministic bound)}
    $\; e^{-n\eta} \;\leq\; \gamma_{i,t} \;\leq\; e^{n\eta}$.
    \item \emph{(R\'{e}nyi-2 variance bound)}
    $\;
    \mathbb{E}_{Q_{i,t}}[\gamma_{i,t}^2 \mid y_{i,<t}]
    = e^{D_2(P_{i,t}\,\|\,Q_{i,t})}
    \leq e^{2n\eta}$,
    hence
    $\mathrm{Var}_{Q_{i,t}}[\gamma_{i,t} \mid y_{i,<t}]
    \leq e^{2n\eta} - 1$,
\end{enumerate}
where $n$ is the number of tokens in the guidance text added by $g(x)$, $Q_{i,t}=\pi_{\theta_{\mathrm{old}}}(\cdot \mid g(x), y_{i,<t})$, and $P_{i,t}=\pi_{\theta_{\mathrm{old}}}(\cdot \mid x, y_{i,<t})$.
\end{theorem}

The key observation is that the variance bound scales with the guidance length $n$, not the rollout length $T$. This also motivates the design principle \emph{use the shortest guidance prefix that restores learning signal}. Our main method~\algonameFixed{} builds on this principle.

\begin{proof}
\textbf{(i)} By a telescoping
decomposition,
\[
\log \gamma_{i,t}
\;=\;
\sum_{k=1}^{n} \log
\frac{\pi_{\theta_{\mathrm{old}}}(y_{i,t} \mid x, u_{1:k-1}, y_{i,<t})}{
\pi_{\theta_{\mathrm{old}}}(y_{i,t} \mid x, u_{1:k}, y_{i,<t})}.
\]
By Assumption~\ref{ass:eta} each summand is in $[-\eta, \eta]$, so
$|\log \gamma_{i,t}| \leq n\eta$. Exponentiating gives
$e^{-n\eta} \leq \gamma_{i,t} \leq e^{n\eta}$.

\textbf{(ii)} Recall that for $P \ll Q$, the exponentiated R\'enyi-2
divergence is $e^{D_2(P\|Q)} = \sum_v P(v)^2/Q(v) =
1 + \chi^2(P\|Q)$. With $P_{i,t} =
\pi_{\theta_{\mathrm{old}}}(\cdot \mid x, y_{i,<t})$ and $Q_{i,t} =
\pi_{\theta_{\mathrm{old}}}(\cdot \mid g(x), y_{i,<t})$, the second
moment under $Q_{i,t}$ is
\[
\mathbb{E}_{Q_{i,t}}[\gamma_{i,t}^2 \mid y_{i,<t}]
= \sum_v Q_{i,t}(v) \big(P_{i,t}(v)/Q_{i,t}(v)\big)^2
= \sum_v P_{i,t}(v)^2 / Q_{i,t}(v)
= e^{D_2(P_{i,t}\|Q_{i,t})}.
\]
By the deterministic bound from~(i),
$\gamma_{i,t} \leq e^{n\eta}$ almost surely under $Q_{i,t}$, hence
$\mathbb{E}_{Q_{i,t}}[\gamma_{i,t}^2] \leq e^{2n\eta}$. Finally,
since $\gamma_{i,t} = P_{i,t}/Q_{i,t}$ is a likelihood ratio under
$Q_{i,t}$, $\mathbb{E}_{Q_{i,t}}[\gamma_{i,t}] = 1$, so
$\mathrm{Var}_{Q_{i,t}}[\gamma_{i,t}] =
\mathbb{E}_{Q_{i,t}}[\gamma_{i,t}^2] - 1 \leq e^{2n\eta} - 1$.
\end{proof}

\section{Toy Example: Guided-Target Misalignment}
\label{app:toy_example}

Consider a single problem $x$
with guided version $g(x)$, and suppose the policy chooses
between two reasoning modes, $\textsc{shortcut}$ and
$\textsc{robust}$, with $\textsc{shortcut}$ chosen with probability
$\theta \in [0, 1]$. The verifier always evaluates the original
problem $x$. The conditioning on $x$ or $g(x)$ below only specifies which
prompt generated the response. Assume
\[
\Pr[r=1 \mid x, \textsc{shortcut}] = 0,
\quad
\Pr[r=1 \mid g(x), \textsc{shortcut}] = 1,
\]
\[
\Pr[r=1 \mid x, \textsc{robust}] = \beta,
\quad
\Pr[r=1 \mid g(x), \textsc{robust}] = \beta,
\]
for some $\beta \in (0, 1)$. The target objective $J(\theta)$ is
the expected verifier reward when the response is generated from the original
prompt $x$, while the guided objective $J^{\mathrm{guide}}(\theta)$ is the
expected verifier reward when the response is generated from the guided prompt
$g(x)$. Since the policy selects $\textsc{shortcut}$ with probability
$\theta$ and $\textsc{robust}$ with probability $1-\theta$, we have:
\[
J(\theta) = \theta \cdot 0 + (1 - \theta) \beta = (1-\theta) \beta,
\qquad
J^{\mathrm{guide}}(\theta) = \theta \cdot 1 + (1-\theta) \beta = \theta + (1-\theta) \beta.
\]
Therefore
$\arg\max_{\theta} J(\theta) = 0$ and
$\arg\max_{\theta} J^{\mathrm{guide}}(\theta) = 1$. The two objectives
prefer extreme opposite policies. Furthermore,
$J^{\mathrm{guide}}(\theta) - J(\theta) = \theta \geq 0$, so the
guided objective is uniformly larger but rewards a deployment-
useless behavior.
\section{Experimental Details}
\label{app:experimental_details}

\paragraph{Dataset construction.}
We construct our training set from the MATH dataset~\citep{hendrycksmath2021}, restricting to problems at difficulty Levels~3--5 (intermediate through advanced), which yields $5{,}586$ problems from the training split. We identify problems on which the base model $\pi_{\mathrm{ref}}$ receives no learning signal under standard GRPO.
For each problem $x$, we generate $M = 64$ independent rollouts from the model under training using nucleus sampling 
, and retain only those problems for which \textit{every} rollout fails. This filter yields $\mathbf{595}$ hard problems for the 7B model. Similarly, for the 3B and 1.5B models, we repeat this filtering process to identify model-specific hard problems. For fixed guidance methods (including our~\algonameFixed{}), we construct an augmented training set $\mathcal{D}_{\mathrm{aug}}$ as described in Section~\ref{sec:main_results}.
For adaptive guidance methods (including \algonameAdaptive{}), no dataset augmentation is performed.
 
\paragraph{RL training configuration.}
 
All models are trained with Group Relative Policy Optimization (GRPO)~\citep{shao2024deepseekmath} and algorithms derived from GRPO implemented in veRL~\citep{sheng2024hybridflow}, using vLLM~\citep{kwon2023efficient} for rollout generation. We fine-tuned with LoRA~\citep{hu2022lora} adapters (rank $64$, $\alpha = 128$, dropout $0.05$) applied to all linear layers. The base model weights are frozen; only LoRA parameters are updated. We use AdamW with learning rate $1 \times 10^{-5}$, weight decay $0.01$, and gradient clipping at norm $1.0$. Training runs for $4$ epochs on the training set. The batch size is $32$ prompts with $N = 16$ rollouts per prompt, PPO mini-batch size $32$, and $1$ PPO epoch per GRPO iteration. We apply the clipped surrogate objective with clip ratio $\epsilon = 0.2$, remove the KL penalty ($\beta = 0$) following recent RLVR practice~\citep{zhang2025scaf,qu2026pope, yu2025dapo}, and use no entropy bonus. Losses are aggregated at the token level. Importance correction is applied at the token level. Training rollouts use temperature $0.7$, top-$p = 0.95$
. We use a binary verifiable reward: $1$ if the model's extracted answer (from \texttt{\textbackslash boxed\{\}}) matches the ground-truth answer under symbolic equivalence, and $0$ otherwise.
 
\paragraph{Fair comparison across methods.}
For any given model, all methods in our experiments are \emph{epoch-matched}: each is trained for $4$ epochs over the same training datasets: $\mathcal{D}_{\mathrm{aug}}$ (for the fixed-guidance methods) and $\mathcal{D}$ (for the adaptive-guidance methods). 

 
\paragraph{Infrastructure.}
Adaptive-guidance runs use $4 \times$ NVIDIA A100 40GB GPUs with FSDP (CPU parameter and optimizer offloading) and tensor parallelism $2$ for vLLM rollouts.
Fixed-guidance and vanilla runs use $2 \times$ A100 GPUs with tensor parallelism $1$.
vLLM GPU memory utilization is set to $0.5\text{--}0.6$.
 
\paragraph{Seeds.}
All 7B model configurations are trained with $3$ random seeds ($0$, $123$, $456$), and results are reported as mean $\pm$ standard deviation across seeds.
 
\paragraph{Baseline methods.}
We compare our methods, \algonameFixed{} and \algonameAdaptive{}, against four baselines: vanilla GRPO, POPE~\citep{qu2026pope}, PrefixRL~\citep{setlur2026reuse}, and BREAD~\citep{zhangbread}. Note that all the guided baselines have some version of filtering for hard problems as well. For a fair and uniform comparison we implement all of them under the same training dataset setting as ours. We re-implement the core algorithmic idea of each method within our veRL training loop, using the same base model, LoRA configuration, optimizer settings, and evaluation protocol as our own methods, so that differences in final performance reflect algorithmic choices rather than infrastructure differences.


\paragraph{Metrics.}
For each checkpoint and benchmark, we generate $16$ trajectories per problem 
and report pass@$1$ as the primary metric, with pass@$16$ reported alongside to assess the breadth of the solution distribution.
Pass@$k$ is computed using the unbiased estimator formula.
 
\paragraph{Model selection.}
Checkpoints are evaluated at each epoch boundary. For each seed, we select the fourth epoch checkpoint and report that checkpoint's results across all benchmarks. We selected 4 checkpoints for training as we observed the training gains on held out validation taper beyond the fourth checkpoint.

\section{Cascading Hints}\label{app:hint_families}

In this section we consider 5 levels of hints, in contrast with solution prefixes. Algorithm~\ref{alg:csg_rollouts} describes the implementation of cascading hints style guidance. Appendix~\ref{apx: hint template} contains the prompt templates for generating hints. Unlike solution prefixes, guidance with the hints is a two step process: (a) generation of the hint; (b) constructing a guided prompt using hints. 

\begin{table*}[h]
\centering
\scriptsize
\setlength{\tabcolsep}{3.2pt}
\renewcommand{\arraystretch}{1.08}
\resizebox{\textwidth}{!}{%
\begin{tabular}{lcccccccccccc}
\toprule
\textbf{Method}
& \multicolumn{2}{c}{\textbf{Training}}
& \multicolumn{2}{c}{\makecell{\textbf{AIME}\\\textbf{(1983--2026)}}}
& \multicolumn{2}{c}{\textbf{Gaokao2023}}
& \multicolumn{2}{c}{\textbf{OlympiadBench}}
& \multicolumn{2}{c}{\textbf{Minerva}}
& \multicolumn{2}{c}{\textbf{OmniMath}} \\
\cmidrule(lr){2-3}\cmidrule(lr){4-5}\cmidrule(lr){6-7}\cmidrule(lr){8-9}\cmidrule(lr){10-11}\cmidrule(lr){12-13}
& \textbf{P@1} & \textbf{P@16}
& \textbf{P@1} & \textbf{P@16}
& \textbf{P@1} & \textbf{P@16}
& \textbf{P@1} & \textbf{P@16}
& \textbf{P@1} & \textbf{P@16}
& \textbf{P@1} & \textbf{P@16} \\
\midrule

\rowcolor{gray!10}
\multicolumn{13}{l}{\textbf{Baselines}} \\

Base
& 0.4 & 4.7
& 16.1 & 37.5
& 42.0 & 64.7
& 28.1 & 48.2
& 26.0 & 45.2
& 20.8 & 44.8 \\

Vanilla GRPO
& \rescell{0.7$\pm$0.2}{ref.} & \rescell{7.2$\pm$2.0}{ref.}
& \rescell{16.8$\pm$0.3}{ref.} & \rescell{38.5$\pm$1.0}{ref.}
& \rescell{46.1$\pm$0.7}{ref.} & \rescell{65.9$\pm$0.5}{ref.}
& \rescell{29.6$\pm$0.1}{ref.} & \rescell{48.7$\pm$0.1}{ref.}
& \rescell{26.5$\pm$0.1}{ref.} & \rescell{45.0$\pm$1.1}{ref.}
& \rescell{22.6$\pm$0.2}{ref.} & \rescell{45.8$\pm$0.6}{ref.} \\

\midrule
\rowcolor{gray!10}
\multicolumn{13}{c}{\textbf{Fixed Hints}} \\

\algonameFixed{} (Hints)
& \rescell{\textbf{2.7$\pm$1.2}}{(+276.8\%)} & \rescell{\textbf{11.7$\pm$2.6}}{(+62.0\%)}
& \rescell{16.6$\pm$2.6}{(-1.3\%)} & \rescell{\textbf{40.5$\pm$1.2}}{(+5.1\%)}
& \rescell{46.5$\pm$6.0}{(+0.9\%)} & \rescell{67.0$\pm$2.4}{(+1.7\%)}
& \rescell{28.4$\pm$2.4}{(-4.3\%)} & \rescell{48.3$\pm$0.7}{(-0.7\%)}
& \rescell{26.7$\pm$2.2}{(+0.6\%)} & \rescell{45.5$\pm$0.7}{(+1.1\%)}
& \rescell{22.0$\pm$2.2}{(-2.8\%)} & \rescell{44.6$\pm$0.3}{(-2.5\%)} \\

\addlinespace[4pt]
\midrule
\rowcolor{gray!10}
\multicolumn{13}{c}{\textbf{Adaptive Hints}} \\

\algonameAdaptive{} (Hints)
& \duringbest{1.4$\pm$0.2}{(+91.3\%)} & \duringbest{7.8$\pm$0.5}{(+8.5\%)}
& \duringbest{\textbf{17.9$\pm$0.2}}{(+6.4\%)} & \rescell{39.8$\pm$0.2}{(+3.4\%)}
& \duringbest{\textbf{49.9$\pm$0.1}}{(+8.3\%)} & \duringbest{\textbf{67.6$\pm$0.7}}{(+2.6\%)}
& \duringbest{\textbf{30.0$\pm$0.3}}{(+1.1\%)} & \rescell{48.7$\pm$0.3}{(+0.1\%)}
& \duringbest{\textbf{27.7$\pm$0.5}}{(+4.4\%)} & \duringbest{\textbf{46.2$\pm$1.1}}{(+2.7\%)}
& \duringbest{\textbf{23.2$\pm$0.2}}{(+2.7\%)} & \duringbest{\textbf{46.2$\pm$0.9}}{(+0.9\%)} \\

\algoname{} (Self-Correction)
& \rescell{0.7$\pm$0.1}{(-2.4\%)} & \rescell{6.9$\pm$0.4}{(-4.7\%)}
& \rescell{17.3$\pm$0.2}{(+3.1\%)} & \duringbest{40.3$\pm$0.9}{(+4.5\%)}
& \rescell{46.4$\pm$2.1}{(+0.6\%)} & \rescell{66.5$\pm$1.0}{(+0.9\%)}
& \rescell{29.8$\pm$0.4}{(+0.7\%)} & \rescell{49.2$\pm$1.0}{(+1.0\%)}
& \rescell{27.0$\pm$0.9}{(+1.9\%)} & \rescell{45.2$\pm$0.8}{(+0.5\%)}
& \rescell{22.7$\pm$0.6}{(+0.4\%)} & \rescell{45.9$\pm$0.2}{(+0.4\%)} \\

\algonameAdaptive{} (Frontier Hints)
& \rescell{0.8$\pm$0.3}{(+6.8\%)} & \rescell{6.2$\pm$1.6}{(-14.7\%)}
& \rescell{17.4$\pm$0.5}{(+3.4\%)} & \rescell{40.1$\pm$1.5}{(+4.0\%)}
& \rescell{48.5$\pm$1.4}{(+5.2\%)} & \rescell{67.2$\pm$0.2}{(+2.0\%)}
& \rescell{29.9$\pm$0.4}{(+0.8\%)} & \duringbest{\textbf{49.4$\pm$0.4}}{(+1.4\%)}
& \rescell{27.5$\pm$0.5}{(+3.6\%)} & \rescell{45.5$\pm$0.6}{(+1.1\%)}
& \rescell{23.0$\pm$0.6}{(+1.6\%)} & \rescell{45.7$\pm$0.4}{(-0.1\%)} \\
\bottomrule
\end{tabular}%
}
\caption{Hint ablation on Qwen2.5-7B-Instruct using the best checkpoint per seed selected by validation Pass@1. Each non-base cell reports seed-averaged mean$\pm$std on the first line and relative improvement over Vanilla GRPO on the second line. Within the Adaptive Hints section, the best method per metric is shown in italics. Overall best run(s) for each metric are bolded.}
\label{tab:hint_ablation_qwen7b_bestepoch}
\end{table*}

\begin{algorithm}[h]
\caption{Cascading self-guidance rollout construction (per training step)}
\label{alg:csg_rollouts}
\KwIn{Batch $\mathcal{B}=\{x^{(d)}\}_{d=1}^{D}\subset\mathcal{D}_{\mathrm{hard}}$; snapshot policy $\pi_{\theta_{\mathrm{old}}}$; verifier $v$; group size $N$; max hint level $L$.}
\KwOut{Final rollout groups $\{G^{(d)}_{\mathrm{final}}\}_{d=1}^{D}$ and behavior contexts $\{b^{(d)}\}_{d=1}^{D}$.}

\ForEach{$x \in \mathcal{B}$}{
  $G \leftarrow \emptyset$\;
  \For{$i\leftarrow 1$ \KwTo $N$}{
    sample $y_i \sim \pi_{\theta_{\mathrm{old}}}(\cdot\mid x)$\;
    $G \leftarrow G \cup \{y_i\}$\;
  }
  compute rewards $\{r_i\}_{i=1}^{N} \leftarrow v(x, G)$\;
  \If(\tcp*[f]{At least one success under the original prompt}){$\sum_{i=1}^{N} r_i > 0$}{
    $G_{\mathrm{final}} \leftarrow G$; $b \leftarrow x$.
  }
  \Else(\tcp*[f]{Learning cliff: all $N$ rollouts fail}) {
    $(G_{\mathrm{final}},b) \leftarrow \textsc{RegenerateWithSelfGuidance}(x, G, \pi_{\theta_{\mathrm{old}}}, v, N, L)$\;
  }
  store $(G^{(d)}_{\mathrm{final}}, b^{(d)}) \leftarrow (G_{\mathrm{final}}, b)$\;
}
\Return{$\{(G^{(d)}_{\mathrm{final}},b^{(d)})\}_{d=1}^{D}$}\;
\end{algorithm}

\begin{algorithm}[h]
\caption{\textsc{RegenerateWithSelfGuidance}: adjusting hint levels until success}
\label{alg:regen_with_hints}
\KwIn{Problem $x$; failed unguided group $G=\{y_i\}_{i=1}^{N}$; snapshot policy $\pi_{\theta_{\mathrm{old}}}$; verifier $v$; group size $N$; max hint level $L$.}
\KwOut{A final group $G_{\mathrm{final}}$ and its behavior context $b\in\{g_\ell(x)\}_{\ell=1}^{L} \cup \{x\}$.}

\For{$\ell \leftarrow 1$ \KwTo $L$}{
  generate hint $h_\ell \leftarrow \textsc{HintGen}_\ell(x, G)$\tcp*[f]{L1--L4 use failures; L2--L5 may use $y^\star(x)$}
  form guided prompt $g_\ell(x) \leftarrow (x,h_\ell)$\;
  $G' \leftarrow \emptyset$\;
  \For{$i\leftarrow 1$ \KwTo $N$}{
    sample $y'_i \sim \pi_{\theta_{\mathrm{old}}}(\cdot\mid g_\ell(x))$\;
    $G' \leftarrow G' \cup \{y'_i\}$\;
  }
  compute rewards $\{r'_i\}_{i=1}^{N} \leftarrow v(x, G')$\;
  \If(\tcp*[f]{Stop as soon as any guided rollout succeeds}){$\sum_{i=1}^{N} r'_i > 0$}{
    \Return{$(G', g_\ell(x))$}\;
  }
}
\Return{$(G, x)$}\tcp*[f]{Fallback: no success at any hint level}\;
\end{algorithm}

\section{All Prompt Templates}

\subsection{Solution Prefix as Guidance}\label{apx: solution prefix template}

In place of LLM-generated hints, the prefix cascade exposes increasing portions of the
reference solution itself as the privileged scaffold. Let $y^\star(x)$ denote the reference
solution after stripping non-text artifacts (e.g., Asymptote diagram blocks). We define a
prefix operator $\pi_\alpha(y^\star)$ that returns the first $\alpha$ fraction of $y^\star$
by character count, snapped to the nearest preceding word boundary, and we instantiate the
cascade at fractions $\alpha \in \{0.2,\, 0.4,\, 0.6,\, 0.8,\, 1.0\}$. At $\alpha < 1.0$ the
prefix is a strict initial segment of the reference solution that the model is asked to
\emph{complete}; at $\alpha = 1.0$ the prefix is the full reference solution and the model
is asked to \emph{verify and reproduce} it. Unlike the hint cascade, no LLM call is required
to produce a prefix---the operator $\pi_\alpha$ is a pure string operation on the
ground-truth solution, so the cascade is deterministic and free.

The cascade is applied identically in the adaptive and fixed settings. In both cases, prefix-augmented prompts are used only to generate
trajectories; the policy gradient is taken with respect to the original unprefixed prompt under an IS correction (Section~\ref{sec:oc_grpo}).

Prefixes are then injected into a \emph{solve prompt} for trajectory generation. The key
design is that the model is always asked to produce a full solution and a final boxed
answer, while the prefix length varies by level.

\begin{promptbox}[Prefix-conditioned solve prompt (used for guided rollout generation)]
(partial prefix levels) Problem: {question}
Partial reference solution: {prefix}
Complete the rest of the solution step by step and output the final answer within \boxed{}. You can use the partial reference solution that follows the question to help you solve the problem.

system prompt for partial prefix levels: You are a math problem solver. When given a partial reference solution, use it to guide your reasoning, but write your solution independently. Do not repeat, copy, or paraphrase the reference text. Show your own step-by-step work and arrive at the answer through your own reasoning.

(full solution level) Problem: {question}
Reference solution: {solution}
Rephrase the solution above in your own words. Show your step-by-step work and output the final answer within \boxed{}.
\end{promptbox}

The two framings are kept distinct on purpose. For partial prefixes ($\alpha<1$) the
anti-repetition system prompt discourages the model from merely echoing the given segment
and pushes it to produce its own continuation, so that the trajectory contains genuine
model-generated reasoning over which the policy gradient can be computed. For the full
solution ($\alpha=1$) we drop the anti-repetition prompt and ask explicitly for a paraphrase, the goal at
this level is to provide a guaranteed-correct trajectory whose answer-extraction reward is
$1$, anchoring the IS-corrected gradient even on problems the base policy never solves on
its own.

\subsection{Cascading Hint Templates and Hinted Solve Prompts}\label{apx: hint template}

We employ a \emph{hierarchy} of hint generators with increasing privileged information.
Let $y^\star(x)$ denote the reference solution and let $y$ denote a failed model trajectory. We generate hints at levels:
(i) feedback on the failed attempt, (ii) conceptual comparison to the reference, (iii) high-level strategy, (iv) detailed
guidance, and (v) full-solution paraphrase.
\begin{promptbox}[Cascading hint generation templates (L1--L5)]
L1 (blind; uses failure only):
You are an expert math tutor. A student attempted this problem but got it wrong.

Problem:
{problem}

Student's failed attempt:
{trajectory}

Analyze the student's work and think about the concepts in the problem. Provide a guidance or hint that redirects the student towards the correct approach.

Hint:

L2 (conceptual; uses reference + failure):
You are an expert math tutor. A student attempted this problem but got it wrong.

Problem:
{problem}

Reference solution:
{solution}

Student's failed attempt:
{trajectory}

Compare the student's approach with the correct solution and think about the key concepts that the student missed and provide guidance or hint that redirects the student to the correct solution. Do not reveal full solution.

Hint:

L3 (strategic; uses reference):
You are an expert math tutor. A student is struggling with this problem.

Problem:
{problem}

Reference solution:
{solution}

Provide a STEP-BY-STEP STRATEGY for solving this problem as a numbered list. Highlight the key approach needed. Do not reveal final answer.

Hint:

L4 (detailed; uses reference):
You are an expert math tutor. A student is struggling with this problem.

Problem:
{problem}

Reference solution:
{solution}

Provide DETAILED GUIDANCE to help the student solve this problem. Include:
- The correct approach with key intermediate calculations
- Specific values and expressions the student needs
- The critical steps that lead to the answer
You may reveal most of the solution path, but do NOT state the final boxed answer.

Hint:

L5 (paraphrase; uses reference):
Rephrase the reference solution in your own words. Show your step-by-step work and output the final answer within \boxed{}.
\end{promptbox}

Hints are then injected into a \emph{solve prompt} for a re-attempt. The key design is that the model is always asked to
produce a full solution and a final boxed answer, while the hint varies by level.

\begin{promptbox}[Hint-conditioned solve prompt (used for re-attempt generation)]
(L1--L4) Problem: {question}

Hint: {hint}

Let's think step by step and output the final answer within \boxed{}. You can use the hint that follows the question to help you solve the problem.

system prompt for L1--L4: You are a math problem solver. When given a hint, use it to guide your reasoning, but write your solution independently. Do not repeat, copy, or paraphrase the hint text. Show your own step-by-step work and arrive at the answer through your own reasoning.

(L5) Problem: {question}

Full solution: {solution}

Rephrase the solution above in your own words. Show your step-by-step work and output the final answer within \boxed{}.
\end{promptbox}